\documentclass[lettersize,journal]{IEEEtran}
\usepackage[
    debug,
    version=1,
    shorthands,
    ]{preamble}

\title{Kolmogorov-Arnold causal generative models}
\author{
  \centerline{Alejandro Almodóvar \orcidlink{0009-0006-0900-4026}
  \quad Mar Elizo \orcidlink{0009-0006-0703-4387}
  \quad Patricia A. Apellániz \orcidlink{0000-0002-8604-9758} \quad Santiago Zazo \orcidlink{0000-0001-9073-7927} \quad Juan Parras \orcidlink{0000-0002-7028-3179}}
}

\begin{document}
\maketitle

\begin{abstract}
Causal generative models provide a principled framework for answering observational, interventional, and counterfactual queries from observational data. However, many deep causal models rely on highly expressive architectures with opaque mechanisms, limiting auditability in high-stakes domains.
We propose \ours, a causal generative model for mixed-type tabular data where each structural equation is parameterized by a Kolmogorov--Arnold Network (KAN). This decomposition enables direct inspection of learned causal mechanisms, including symbolic approximations and visualization of parent–child relationships, while preserving query-agnostic generative semantics.
We introduce a validation pipeline based on distributional matching and independence diagnostics of inferred exogenous variables, allowing assessment using observational data alone. Experiments on synthetic and semi-synthetic benchmarks show competitive performance against state-of-the-art methods. A real-world cardiovascular case study further demonstrates the extraction of simplified structural equations and interpretable causal effects.
These results suggest that expressive causal generative modeling and functional transparency can be achieved jointly, supporting trustworthy deployment in tabular decision-making settings.
\end{abstract}

\section{Introduction}
\label{sec:introduction}

\IEEEPARstart{M}{achine} learning systems are increasingly deployed in high-stakes domains such as personalized medicine~\citep{KentBMJ2018,sanchez2022causal}, public policy~\citep{Imai_Strauss_2011}, and economics~\citep{Manski2004STRHP}, where decisions must be justified, reproducible, and robust to distributional shifts. In these settings, estimating individualized treatment effects and answering counterfactual questions from observational tabular data is central to decision-making. Structural causal models (SCMs) provide a formal framework for such reasoning, allowing the computation of observational, interventional, and counterfactual quantities under explicit assumptions~\citep{pearl2009causality,peters2017elements}. Recovering a reliable approximation of the underlying SCM is therefore essential for principled policy targeting and individualized care~\citep{Curth2024Individualize}.

However, predictive accuracy alone is not sufficient for deployment in sensitive applications. Regulatory frameworks such as the General Data Protection Regulation (GDPR Art.~22)~\citep{GDPR2016}, the European Union Artificial Intelligence Act~\citep{EUAIACT2024}, and guidance from the U.S. Food and Drug Administration on AI-enabled medical devices~\citep{usfda2024} emphasize transparency and accountability requirements for high-risk AI systems. The ``right to explanation'' and related transparency obligations further reinforce the need for models whose internal logic can be inspected and communicated~\citep{goodman2017european}. In clinical practice, limited interpretability remains a major barrier to the adoption of machine learning systems~\citep{TonekaboniMLHC2019,AmannBMC2020}. As argued by~\citet{rudin2019stop}, black-box models are particularly problematic when decisions have substantial real-world consequences and interpretable alternatives are feasible. In the causal setting, interpretability is even more critical: practitioners must understand how interventions propagate through mechanisms, not merely observe predictive correlations.

Recent advances in causal generative models (CGMs) combine deep generative modeling with SCMs, and demonstrate that flexible neural architectures can approximate structural mechanisms while preserving do-operator semantics~\citep{pawlowski2020dscm,sanchez2022vaca, javaloy2023causal, chao2024dcm}. These approaches are query-agnostic: once trained, they can answer arbitrary observational, interventional, and counterfactual queries without retraining. Despite these advances, most existing CGMs rely on highly expressive neural networks whose learned structural equations are opaque. While the causal graph itself provides structural interpretability, the functional form of each mechanism remains difficult to audit, simplify, or communicate. 

\begin{figure}[t]
\centering
\includegraphics[width=0.9\linewidth]{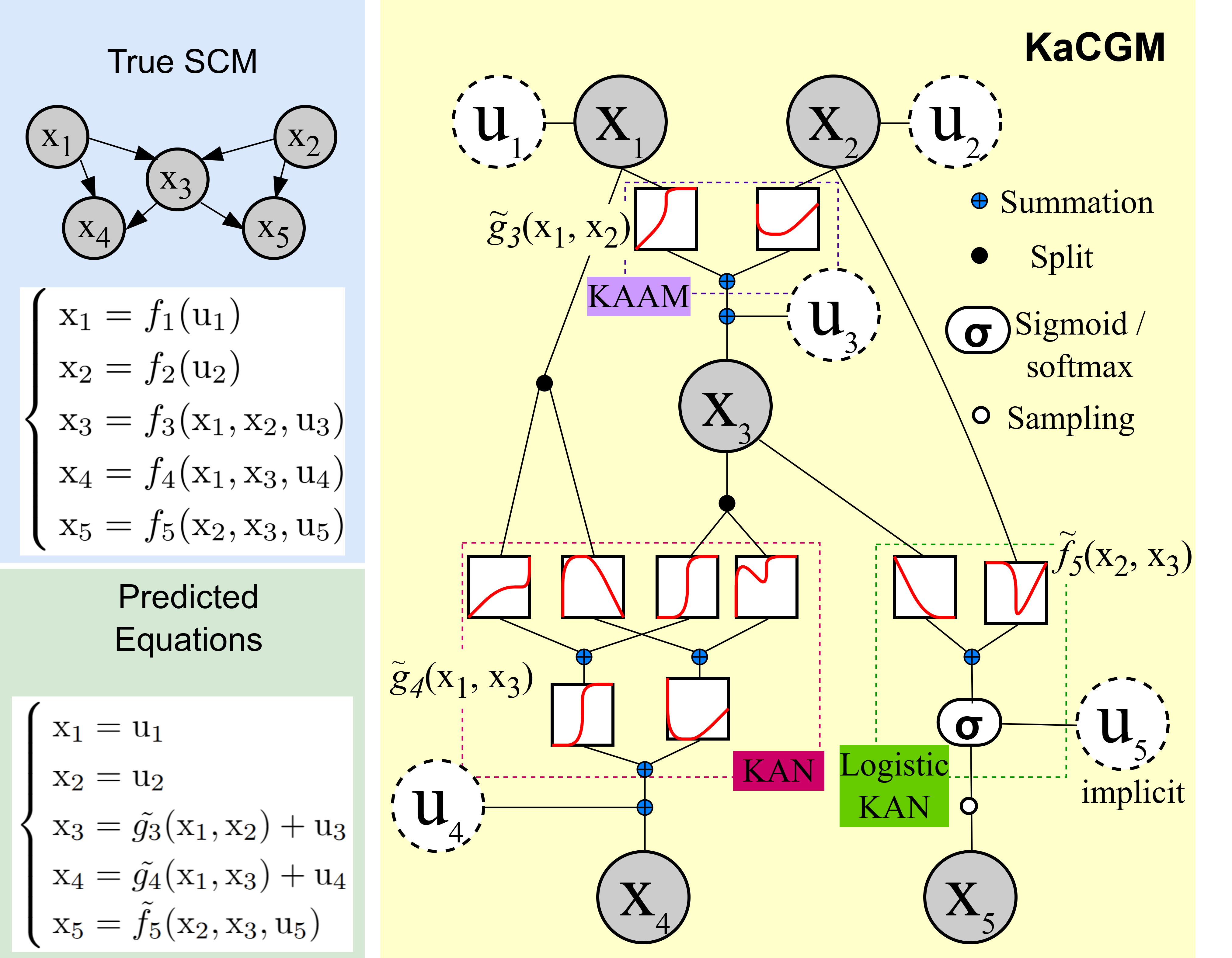}
\caption{Compacted sketch of the proposed causal generative model, \ours, including a discrete variable ($\covariatei_5$). The true SCM, with unknown structural equations (left), is approximated with additive noise models in which the functions of each endogenous variable are modeled by KANs (right). Several KANs are shown, in which interpretability capabilities decrease with network complexity. As explained in \cref{sec:method}, $\exogenous_5$ is modeled implicitly by a logistic-KAN. Closed-form expressions can be extracted from every approximated function.}
\label{fig:kan-scm}
\end{figure}

This work addresses the gap between expressive causal generative modeling and functional interpretability in tabular settings. We propose
\emph{\ourslong} (\ours),
a causal generative model in which each structural mechanism is parameterized by a Kolmogorov-Arnold network (KAN). KANs replace scalar weights with learnable univariate functions, inducing structured representations that facilitate inspection, pruning, and symbolic approximation. By embedding KANs within an additive noise SCM framework, our approach preserves generative capabilities and query-agnostic inference, while supporting the extraction of explicit structural equations.

A compact overview of the proposed model is shown in \cref{fig:kan-scm}. Each endogenous variable is modeled as an additive mechanism with exogenous noise, where the structural function is parameterized by a KAN. Mixed data types, including categorical variables, are supported via appropriate likelihood parameterizations. Crucially, closed-form expressions can be extracted from each learned mechanism, enabling direct auditing and visualization of causal effects.

Our contributions are threefold. First, we introduce a KAN-based causal generative model for mixed-type tabular data that preserves SCM semantics and supports observational, interventional, and counterfactual reasoning in a query-agnostic manner. Second, we provide a validation pipeline based on observational distribution matching and independence testing of inferred exogenous noise, enabling practitioners to assess model misspecification from observational data alone. Third, we develop an automated interpretability framework that prunes structural mechanisms, performs symbolic approximation, and extracts closed-form expressions and visualization tools for effect auditing. Code and an interactive interface for visualizing causal effects are publicly available at \href{\codelink}{\codeshort} and \href{\interfacelink}{\interfaceshort}, respectively.

The remainder of this paper is organized as follows. In \cref{sec:related_work}, we review prior work on tabular generative modeling and deep causal generative models. In \cref{sec:background}, we recall the formal framework of structural causal models and causal queries. In \cref{sec:method}, we introduce \ours, including its additive formulation, mixed-data extension, and interpretability pipeline. In \cref{sec:experiments}, we evaluate the proposed approach on synthetic, semi-synthetic, and real-world datasets, including misspecification analyses and effect visualization. Finally, \cref{sec:conclusion} concludes with a discussion of limitations and future research directions.

\section{Related work}
\label{sec:related_work}

This work lies at the intersection of generative modeling, causal inference, and interpretable neural architectures. Deep generative models such as generative adversarial networks (GANs) and variational autoencoders (VAEs) have achieved remarkable success in modeling complex data distributions~\citep{goodfellow2020generative, kingma2013auto}. However, these models typically target the observational joint distribution and lack causal semantics, which limits their ability to answer interventional or counterfactual queries from observational data alone~\citep{pearl2009causality}.

Causal generative models (CGMs) address this limitation by combining deep generative modeling with the formalism of SCMs, enabling the approximation of data-generating mechanisms while preserving the interventional structure of the SCM~\citep{pearl2009causality}. Classical identifiability results for nonlinear additive noise models (ANMs)~\citep{hoyer2009anm} motivate many recent deep-learning extensions that employ flexible neural architectures to parameterize structural mechanisms and infer exogenous noise variables. Recent approaches span several generative families, including normalizing-flow-based models such as \methodname{causal flows}~\citep{javaloy2023causal}, \methodname{CAREFL}~\citep{khemakhem2021caf}, and \methodname{DCSM}~\citep{pawlowski2020dscm}, as well as related developments~\citep{parafita2022estimand, nasr2023counterfactual, balgi2025deep, almodovar2025decaflow}; GAN-based causal generators such as \methodname{NCM}~\citep{xia2023neural} and other frameworks~\citep{van2021decaf, xu2019achieving, Goudet2018}; diffusion-based approaches such as \methodname{DBCM}~\citep{chao2024dcm}; and variational models such as \methodname{VACA}~\citep{sanchez2022vaca}, among other neural causal modeling approaches~\citep{xia2021causal, sick2025interpretable, stegle2010probabilistic}. Some of these works further establish theoretical guarantees or identifiability results under additional assumptions, including settings with hidden confounders~\citep{parafita2022estimand, nasr2023counterfactual, almodovar2025decaflow, xia2023neural}. While these models provide expressive approximations of SCMs and support query-agnostic inference, the structural mechanisms are typically parameterized by deep neural networks whose functional form is difficult to inspect or audit.

On the other hand, interpretability is a key requirement for deployment in sensitive domains, where models must support auditing, debugging, and transparent decision-making~\citep{doshivelez2017rigorous}. In tabular prediction tasks, interpretable model families such as generalized additive models and neural additive models have been widely studied~\citep{caruana2015intelligible, agarwal2021nam}. In causal inference, interpretability is further tied to the credibility of estimated effects and the ability to justify “what-if” decisions in terms of mechanisms rather than correlations~\citep{pearl2009causality}. Some approaches explore interpretable causal regression through rule-based models~\citep{bargaglistoffi2024causalruleensembleinterpretable} or interpretable surrogate models~\citep{kim2021learninginterpretablemodelscausal}. Nevertheless, although the causal graph provides structural insight, the functional form of the mechanisms learned by causal generative models remains largely opaque.

Recently, Kolmogorov–Arnold networks (KANs) have been proposed as an alternative neural architecture in which scalar weights are replaced by learnable univariate functions defined on edges~\citep{liu2025kan, liu2024kan2}. This representation induces structured functional decompositions that facilitate inspection, pruning, and symbolic approximation. KANs have rapidly been adapted to multiple architectures, such as convolutional networks~\citep{bodner2025convkan} or graph neural networks~\citep{li2025kolmogorov},
and have demonstrated promising interpretability capabilities in applications such as healthcare~\citep{almodovar25kaam, pendyala25effectiveness}, finance~\citep{CHO2025128781}, and biology~\citep{alharbi25biomarker}. Their use in causal inference has only recently begun to be explored, for instance, in treatment effect estimation~\citep{mehendale2025kanite, almodovar2025causalkans}.

Building on these developments, this work investigates the use of KANs as structural mechanisms within causal generative models, aiming to combine the flexibility of deep causal generators with interpretable representations of the learned causal mechanisms.

\section{Background}
\label{sec:background}

\subsection{Structural Causal Models}
\label{subsec:scm}

Let $\covariates = \{\covariatei_1, \covariatei_2, ..., \covariatei_{\covsize}\} \in \covspace$ be a set of random variables with joint distribution $\distribution[\covariates]$.

\begin{definition}
A causal graph $\graph = (\indexset, \edges)$ consists of a set of indices $\indexset = \{1,2,...,\covsize\}$ and directed edges $\edges$, where $(\indexcovariate,\indexthree)\in\edges$ denotes that $\covariatei_\indexcovariate$ causally influences $\covariatei_\indexthree$.
\end{definition}

Let $\parents{\indexcovariate}$ denote the set of parents of node $\indexcovariate$. If the graph is Markov and faithful~\citep{peters2011causal}, the joint distribution factorizes according to the parental Markov condition~\citep{pearl2009causality}:

\begin{equation}
\pdf(\covariates)=\prod_{\indexcovariate=1}^{\covsize}
\pdf(\giventhat{\covariatei_\indexcovariate}{\parents{\indexcovariate}}).
\end{equation}

\begin{definition}
A Structural Causal Model (SCM) $\scm=(\funcb,\distribution[\exogenous])$ defines a data-generating process (DGP) where each observed variable is generated from its parents and an exogenous noise variable $\exogenous=\{\exogenousi_1,...,\exogenousi_\covsize\}$:
\begin{equation}
\covariatei_\indexcovariate =
\func[\indexcovariate](\parents{\indexcovariate},\exogenousi_\indexcovariate),
\quad \exogenous \sim \distribution[\exogenous].
\end{equation}
\end{definition}

An SCM characterizes the full causal data-generating process and therefore determines the value of any causal query~\citep{peters2017elements}. Interventions are formalized through the $\doop$-operator.

\begin{definition}[\doop-operator~\citep{pearl2009causality}]
Let $\intervention\subseteq\indexset$. The intervention $\doop(\covariatei_{\intervention}=a_{\intervention})$ replaces the structural assignments of the intervened variables by constants, while leaving the remaining mechanisms unchanged.
\end{definition}

Causal queries are functionals of the SCM involving at least one intervention, including interventional distributions, causal effects, and counterfactual quantities~\citep{pearl2009causality,peters2017elements}. SCMs therefore support reasoning at the three levels of Pearl's hierarchy: observational, interventional, and counterfactual. Interventional distributions are obtained by sampling $\exogenous\sim\distribution[\exogenous]$ and propagating the modified structural assignments, while counterfactuals are computed via the abduction--action--prediction procedure~\citep{pearl2009causality,peters2017elements}.

\subsection{Causal generative models}
\label{subsec_gcm}

Causal generative models (CGMs) are generative machine learning models that respect a causal graph $\graph$ and approximate the mechanisms of an SCM $\scm$. A learned model $\gcm=(\pred{\funcb},\pred{\distribution[\exogenous]})$ can then be used as a simulator of the underlying causal system.

Once the structural mechanisms are approximated, the model supports observational, interventional, and counterfactual reasoning through the usual $\doop$-operator semantics. For continuous variables, under causal sufficiency and with sufficiently expressive function classes, CGMs can recover the structural equations up to standard identifiability considerations~\citep{javaloy2023causal}, leading to consistent answers to causal queries~\citep{pearl1999probabilities,pearl2009causality}.

CGMs are typically \emph{query-agnostic}: once trained, they allow answering arbitrary causal queries without retraining~\citep{parafita2022estimand}. Recent work has proposed flexible neural CGMs with universal approximation capabilities for SCMs, including models based on normalizing flows, diffusion models, and neural structural mechanisms (see \cref{sec:related_work}). However, most existing approaches rely on highly expressive neural networks whose mechanisms are difficult to interpret.

Although neural networks are universal function approximators, certain classes of functions are known to be more efficiently approximated by spline-based representations, particularly smooth functions in low-data regimes~\citep{friedman90adaptive}. This motivates exploring alternative architectures that preserve flexibility while improving interpretability, such as Kolmogorov–Arnold networks.

\subsection{Kolmogorov-Arnold Networks}
\label{subsec:kan}

Kolmogorov–Arnold networks (KANs) provide an alternative representation to deep neural networks. While standard neural networks rely on the universal approximation theorem~\citep{hornik1989multilayer}, KANs are inspired by the Kolmogorov–Arnold representation theorem~\citep{kolmogorov1957representations,Arnold1957}, which states that any continuous function $\func:\mathbb{R}^{\covsize}\rightarrow\mathbb{R}$ can be expressed as

\begin{equation}
\func(\samplecov)=
\sum_{\indexthree=1}^{2\covsize+1}
\kanoutlayer_\indexthree
\Big(
\sum_{\indexfour=1}^{\covsize}
\kanlayer_{\indexthree,\indexfour}(\samplecovi_\indexfour)
\Big),
\end{equation}

where $\kanlayer$ and $\kanoutlayer$ are univariate nonlinear functions.

Modern implementations approximate these functions using stacked layers, similarly to neural networks~\citep{liu2025kan}. A KAN with $\nLayers$ layers can be written as

\begin{equation}
\text{KAN}(\samplecov)=
(\kan_{\nLayers-1}\circ\cdots\circ\kan_0)(\samplecov),
\end{equation}

where each layer applies sums of learnable univariate functions.
In practice, the univariate functions are parameterized with B-splines,

\begin{equation}
\kanlayer(\samplekan)=
\splinebiascoef\,b(\samplekan)+
\splinecoef\,\text{spline}(\samplekan),
\end{equation}

with $b(\cdot)$ a fixed baseline and $\splinecoef,\splinebiascoef$ learnable parameters~\citep{de1978practical}. This formulation allows training with standard optimization methods while maintaining a structured functional representation.

Because KANs explicitly model univariate transformations, they enable direct inspection of nonlinearities and facilitate symbolic approximation and pruning. Additional regularization and sparsification techniques can further simplify the learned functions~\citep{liu2024kan2}. In practice, shallow KAN architectures often achieve a favorable balance between expressiveness and interpretability, which has motivated their use in several scientific and applied domains~\citep{almodovar25kaam,Knottenbelt_2025,CHO2025128781}.
\section{\ourslong}
\label{sec:method}

Assume that we have a dataset, $\dataset = \{\samplecovi_1^{(\indexsample)}, \samplecovi_2^{(\indexsample)}, ..., \samplecovi_\covsize^{(\indexsample)}\}_{\indexsample=1}^\samplesize$, containing \samplesize i.i.d. samples from the unknown observational distribution $\distribution[\covariates]$, generated by an unknown SCM \scm, with a known causal graph \graph. Given that, we want to obtain a causal generative model \gcm that approximates \scm and, in addition, that provides as many insights about each function $\func[\indexcovariate]$ as possible: closed-form equations, interaction terms, symbolic approximations, feature importance, visualization plots of separable effects, etc.

 In the following, we assume \emph{causal sufficiency}, \ie, that there are no hidden confounders, and therefore, 
 all the exogenous variables are independent: $\pdf(\exogenous) = \prod \pdf(\exogenousi_\indexcovariate)$.

Our approach, \ours, consists of several components. First, we define the KAN-based ANM, which is suitable for continuous variables. Second, we extend it to support discrete or categorical variables. Third, we describe techniques to achieve closed-form equations and plotting tools that help to achieve interpretability.

\subsection{\Ours as ANM}
\label{sec:kan_anm}
\ours is an additive noise model (ANM) in which each mechanism of the true SCM is modeled following \cref{eq:anm_equations}, where each function that models the effect of the parents on their children, $\pred{\funcg[\indexcovariate]} (\parents{\indexcovariate})$, is approximated by a KAN.

\begin{equation}
    \pred{\func[\indexcovariate]}(\parents{\indexcovariate}, \exogenousi_\indexcovariate) = \pred{\funcg[\indexcovariate]} (\parents{\indexcovariate}) + \exogenousi_\indexcovariate
    \label{eq:anm_equations}
\end{equation}

The choice of ANM has three main motivations: \itemi in the causal inference literature, there are many relevant papers that assume additive noise models to capture real DGPs
\citep{hoyer2009anm, schultheiss24assesing, shimizu2006linear, buhlmann2014cam, mooij2009regression}
\itemii it enables a very interpretable way to understand the exogenous variables as independent additive noise coming from measurement error, a modeling used in fields such as communications~\citep{proakis1994communication}, robotics~\citep{barfoot2024state} or econometrics~\citep{fuller2009measurement},
and
\itemiii it is reasonably easy to test if our ANM is a good approximator of the true data-generating process.

Regarding the third point, and although proving that the true DGP is additive is not testable, in general, we can test how well specified our ANM is with respect to the observational data~\citep{schultheiss24assesing}. When facing a real-world problem in which the ground-truth of causal effect cannot be calculated, we recommend testing model misspecification, and only use our model if those tests are satisfactory.

\begin{itemize}
    \item Test the independence of the exogenous noise: $\exogenousi_\indexcovariate \indep \parents{\indexcovariate}$. Note that we assume that the causal graph is well specified, therefore if the independence condition is not met, it is because of a model misspecification (heteroscedastic noise
    or nonadditivity), not because the causal direction is reversed~\citep{hoyer2009anm}.
    Following~\citet{hoyer2009anm}, we propose to use the statistic of the Hilbert-Schmidt Independence test (HSIC)~\citep{gretton2005measuring}. In the same fashion, we recommend testing that \exogenous are jointly independent using the multidimensional test dHSIC~\citep{pfister2018kernel}.

    \item Test the observational match. If a more flexible (\eg \methodname{causal flows}~\citep{javaloy2023causal}) model achieves better observational fitting, we consider model misspecification that could mean the DGP is nonadditive. We propose measuring the match of the observational distribution with both \emph{maximum mean discrepancy} (MMD) and classifier two-sample test (C2ST) using the accuracy of a random forest, as explained in \cref{sec:experiments}.
\end{itemize}

Note that meeting observational matching and independence tests do not allow us to confirm that the true SCM is captured by an ANM, but support well-specification in the sense of ``not rejecting misspecification''.

Operation details about training, hyperparameter search, implementation of the do-operator and model misspecification tests can be found in \cref{app:sec:operation}.

\subsection{Mixed type \oursmix}

In tabular data, there are usually discrete or categorical variables. Therefore, we allow our model
to use discrete or categorical models as predictors; in this case, Logistic-KANs. Following the terminology of~\citet{almodovar25kaam}, a Kolmogorov-Arnold additive model (KAAM) is a KAN with no hidden layers, and a Logistic-KAN is a KAN that handles discrete variables by applying a sigmoid function for binary variables or a softmax function for categorical variables. We allow each network within \oursmix to be a KAAM, a Logistic-KAN, a Logistic-KAAM, or a standard KAN.

We show in \cref{prop:interventional_mixed} how modeling the distribution $\pdf_\parameters(\giventhat{\covariatei_\indexcovariate}{\parents{\indexcovariate}})$ of a categorical variable with a flexible enough predictor is equivalent to recovering the function $\func[\indexcovariate](\parents{\indexcovariate}, \exogenousi_\indexcovariate)$ with the exogenous term implicit. Although the exogenous term cannot be identified, all interventional queries in the SCM can.

\begin{proposition}[mixed CGM equivalence with SCMs] Consider the true SCM, \scm, in which one variable, $\covariatei_\indexcovariate$ is discrete, with observational probability mass function $\pdf(\giventhat{\covariatei_\indexcovariate}{\parents{\indexcovariate}})$. Let $\pdf_{\parameters}(\giventhat{\covariatei_\indexcovariate}{\parents{\indexcovariate}})$ be a conditional categorical model (\eg Logistic KAN) such that $\pdf_{\parameters}(\giventhat{\cdot}{\parents{\indexcovariate}}) \rightarrow \pdf(\giventhat{\cdot}{\parents{\indexcovariate}})$. Being $\cdf_\parameters(\giventhat{\cdot}{\parents{\indexcovariate}})$ the cumulative density function (CDF) of $\pdf_\parameters$ and defining
\begin{equation}
    \pred{\func[\parameters]}(\parents{\indexcovariate}, \sampleexogenousi_\indexcovariate) = \min\{k: \sampleexogenousi \leq \cdf_\parameters (\giventhat{k}{\parents{\indexcovariate}})\},
\end{equation}
    with $\exogenousi_\indexcovariate \sim \uniform(0,1)$. Then \itemi $\covariatei_\indexcovariate = \pred{\func[\parameters]}(\parents{\indexcovariate}, \exogenousi_\indexcovariate)$ defines a SCM mechanism whose induced conditional law equals $\pdf_\parameters(\giventhat{\cdot}{\parents{\indexcovariate}})$: sampling from $\pdf_\parameters(\giventhat{\covariatei_\indexcovariate}{\parents{\indexcovariate}})$ is equivalent to sampling from $\exogenousi_\indexcovariate$ and applying $\pred{\func[\parameters]}$; and \itemii the interventional distribution of $\covariatei_\indexcovariate$ and its descendants equals the one of \scm.
    \label{prop:interventional_mixed}
\end{proposition}

 On the other hand, identifiability of point-valued counterfactuals is not possible with discrete outcomes, and only bounds can be provided~\citep{balke1997bounds, balke1994probabilistic, tian2000probabilities}. Fortunately, as stated in \cref{prop:counterfactual_mixed}, in a mixed SCM, there exist counterfactual values that can be recovered, depending on the intervention of interest.

\begin{proposition}[Counterfactual identifiability in mixed GCM] Let the partition $\indexset = \gC \cup \gZ$ be the set of indices of the continuous and discrete variables, respectively. Given the intervention on \intervention, and let \descendants{\intervention} be the descendants of \intervention in \graph. Consider a well-specified causal generative model, \gcm, in which the map $\exogenousi_\indexcovariate \mapsto \pred{\func[\indexcovariate]}(\parents{\indexcovariate}, \exogenousi_\indexcovariate)$ for every $\indexcovariate\in \gC$ is injective, then \gcm yields point-valued counterfactual predictions if there are not discrete variables in the descendants of \intervention, \ie $\descendants{\intervention} \in \gC$.
\label{prop:counterfactual_mixed}
\end{proposition}

Proofs of \cref{prop:interventional_mixed} and  \cref{prop:counterfactual_mixed} can be found in \cref{app:sec:proofs}. Note that both propositions are valid for any CGM that models the distribution of categorical variables.

This adaptation of \ours to \oursmix allows us to increase flexibility when modeling tabular data beyond continuous features, while retaining the identifiability of some causal queries of interest. A sketch of \oursmix can be observed in \cref{fig:kan-scm}, where each function of the true SCM is modeled by a KAN, including the discrete variables, in which the exogenous noise is implicit.

 \begin{figure*}[ht]
    \centering
    \includegraphics[width=0.9\linewidth]{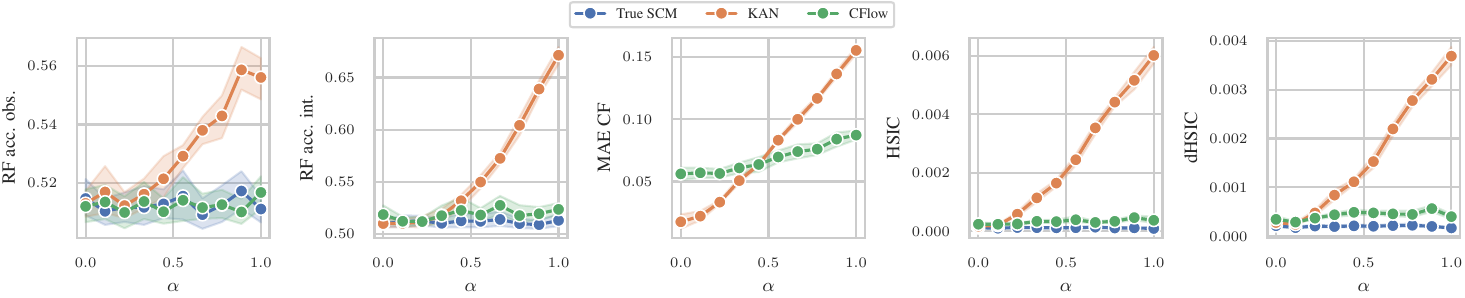}
    \caption{Sensitivity analysis on additivity of exogenous variables. Mean and 95$\%$ confidence intervals over 10 realizations of the DGP. The true SCM represents the metrics obtained by data generated by the true SCM. CFlow, as a more flexible model, achieves good approximations and independence even with non-additive noise, while \ours deviates from this.}
    \label{fig:sensitivity}
\end{figure*}

\subsection{Towards interpretability}

Here, we define some steps that highlight the interpretable capabilities of the KANs, including inductive biases towards simplicity that not only regularize the output to get better approximations of smoother functions, but also improve interpretability:

\begin{enumerate}
    \item Between the best KAN models (fail to reject that they are worse than the best), selecting the one with fewer hidden layers (ideally, a KAAM).
    \item Extract each predicted function $\pred{\func[\indexcovariate]}$.
    \item Prune the edges of the KAN that show little importance, following the process described by~\citet{liu2024kan2}, yielding~$\pred{\func[\indexcovariate]}^{(p)}$. This only removes the non-relevant variables
    \item Fit a symbolic substitution of $\pred{\func[\indexcovariate]}^{(s)}$ following the process of~\citet{almodovar2025causalkans}. This approximates the function to a combination of atoms: polynomials, trigonometric functions, exponentials, etc.
\end{enumerate}

When performing each of these steps, the well-specification tests described in \cref{sec:kan_anm} should be carried out. If the results are not satisfactory, the simplification step \textbf{is not recommended}. One key observation is that, in each of the steps, the distribution of the residuals should be fitted on observational data again; otherwise, the distribution matching will worsen.

Once we have a satisfactory, simplified $\pred{\func[\indexcovariate]}$, we can either \itemi analyze directly the closed-form expression, which is a composition of univariate functions that could be audited if the number of layers of the selected KAN is low; and \itemii plot the contribution of each of the parents to the target variable, extracting feature importance, and using the plotting tools developed by~\citet{almodovar25kaam}: \textit{probability radar plots} (PRP) and \textit{partial dependence plots} (PDP).



Summarizing, the process in a real environment with observational data consists of \itemi constructing and fit \ours in observational data, \itemii testing the observational fit and the independence between exogenous inferred variables, comparing with other more flexible models, \itemiii applying the simplification steps of pruning and symbolic substitution, including tests of step ii), and accept the simplification only if the tests results are not worse up to some threshold established by the user, and \itemiv extracting the closed form expressions of the generating functions and provide interpretable plots.

\begin{table*}[ht]
\scriptsize
\setlength{\tabcolsep}{3.2pt}
\centering
\begin{tabular}{lcccccc|ccccc}
\hline
& \multicolumn{6}{c}{Continuous variables} & \multicolumn{5}{c}{Mixed type variables} \\
\cmidrule(lr){2-7} \cmidrule(lr){8-12}
Model 
& $\text{RF}_\text{acc}^\text{obs}$ 
& $\text{MMD}_\text{obs}$ 
& $\text{RF}_\text{acc}^\text{int}$ 
& $\text{MMD}_\text{int}$ 
& $\text{MAE}_\text{CF}$ 
& p-value 
& $\text{RF}_\text{acc}^\text{obs}$ 
& $\text{MMD}_\text{obs}$ 
& $\text{RF}_\text{acc}^\text{int}$ 
& $\text{MMD}_\text{int}$ 
& p-value \\

& & $\times 10^{-3}$ & & $\times 10^{-3}$ & $\times 10^{-2}$ &
& & $\times 10^{-3}$ & & $\times 10^{-3}$ & \\
\hline

KAN 
& $\mathbf{0.51_{0.0}}$ & $16.49_{8.1}$ & $\mathbf{0.53_{0.0}}$ & $4.79_{4.1}$ & $5.10_{2.4}$ & $0.061^*$
& $\mathbf{0.52_{0.02}}$ & $1.64_{0.74}$ & $0.54_{0.02}$ & $9.11_{7.24}$ & $0.042$ \\

KAAM 
& $0.52_{0.0}$ & $16.76_{9.6}$ & $\mathbf{0.53_{0.0}}$ & $4.76_{5.1}$ & $4.66_{2.3}$ & $0.055^*$
& $\mathbf{0.52_{0.01}}$ & $1.89_{0.93}$ & $0.54_{0.02}$ & $8.40_{5.63}$ & $0.024$ \\

ANM 
& $0.53_{0.0}$ & $\mathbf{8.86_{2.6}}$ & $0.54_{0.0}$ & $\mathbf{3.02_{2.5}}$ & $\mathbf{2.62_{2.5}}$ & $\underline{0.500}^*$
& $0.54_{0.02}$ & $\mathbf{1.22_{0.44}}$ & $0.56_{0.03}$ & $\mathbf{6.22_{5.16}}$ & $0.042$ \\

DBCM 
& $0.56_{0.0}$ & $31.50_{23.0}$ & $0.55_{0.0}$ & $6.79_{5.8}$ & $3.98_{1.5}$ & $<1\text{e-}3$
& $0.58_{0.04}$ & $3.43_{1.91}$ & $0.57_{0.03}$ & $12.53_{9.90}$ & $<1\text{e-}3$ \\

CFlow 
& $0.52_{0.0}$ & $20.46_{6.9}$ & $0.61_{0.1}$ & $9.61_{8.7}$ & $3.86_{1.7}$ & $<1\text{e-}3$
& $0.53_{0.03}$ & $3.16_{1.68}$ & $0.62_{0.09}$ & $7.18_{4.99}$ & $<1\text{e-}3$ \\

$\text{KAN}_\text{mix}$ 
& $-$ & $-$ & $-$ & $-$ & $-$ & $-$
& $\mathbf{0.52_{0.01}}$ & $1.79_{0.92}$ & $\mathbf{0.52_{0.01}}$ & $8.24_{9.88}$ & $\underline{0.690}^*$ \\

$\text{KAAM}_\text{mix}$ 
& $-$ & $-$ & $-$ & $-$ & $-$ & $-$
& $\mathbf{0.52_{0.02}}$ & $1.60_{0.78}$ & $\mathbf{0.52_{0.01}}$ & $6.29_{4.33}$ & $0.690^*$ \\

\hline
\end{tabular}
\caption{
Unified results for continuous and discrete experiments. Metrics report averages over the 11 synthetic datasets ($\text{mean}_\text{std}$) with $1000$ training samples. Lower is better for all metrics. Continuous experiments additionally report counterfactual error ($\text{MAE}_\text{CF}$), which is not defined for the discrete setting. The last column of each block reports p-values from post-hoc Wilcoxon tests aggregated over all metrics. The Friedman test p-value is $<1\text{e-}3$ for both experiments, indicating significant differences among methods overall. Underline denotes the best method. $^*$ indicates no statistically significant difference ($p>0.05$).}
\label{tab:results_unified}
\end{table*}

\begin{figure*}[ht]
    \centering
    \begin{subfigure}{0.65\linewidth}
            \includegraphics[width=\linewidth, trim={0 4cm 0 0}, clip]{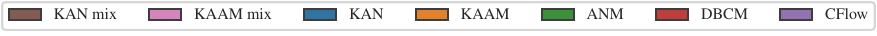}
    \end{subfigure}
    \begin{subfigure}{0.48\linewidth}
    \centering
    \includegraphics[width=.95\linewidth, trim={0 0 0 0.6cm}, clip]{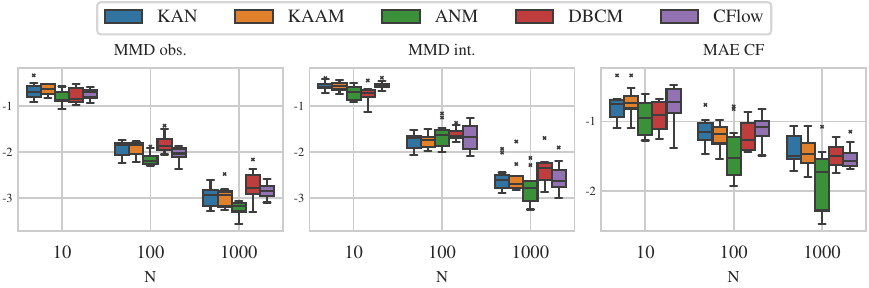}
    \caption{Continuous experiment.}
    \label{fig:rf_obs_continuous_samples}
    \end{subfigure}
    \begin{subfigure}{0.48\linewidth}
            \includegraphics[width=\linewidth, trim={0 0 0 0.65cm}, clip]{figs/discrete_results.pdf}
    \caption{Mixed type experiment.}
    \label{fig:discrete_metrics}
    \end{subfigure}
    \caption{Boxplot of $\log_{10}$ of observational and interventional MMD of each model across 10 realizations with \captiona continuous data and \captionb mixed type data, as a function of the number of samples with mixed type data.}
    \label{fig:synthetic_results}
\end{figure*}

\section{Experiments}
 \label{sec:experiments}

 We have carried out four experiments to support the use of \ours, based on the previous process. This section is organized as follows. First, we present the metrics that we use to measure performance and contrast our hypothesis. Second, we perform a sensitivity analysis where we evaluate how our model suffers from misspecification when the true DGP is not additive, therefore guiding the practitioner through the process in real data. Third, we evaluate the performance of \ours in a set of synthetic datasets, of both continuous and discrete data, compared with state-of-the-art CGM, to show that under well specification, the performance of \ours is competitive for all data regimes, depending on the amount of data available. Fourth, we present results over a semi-synthetic dataset in which we measure the performance on more complex data. Finally, we implement the previous process in a real-world medical dataset, in which we carry out the misspecification tests and extract closed expressions of the estimated functions, also providing interpretable plots that help to understand the causal effects in the whole complex system.

\subsection{Validation metrics}

For evaluating the fitting of the observational distribution and the interventional distributions,  we use both the maximum mean discrepancy (MMD)~\citep{gretton2012kernel} with Gaussian kernel and a classifier two samples test (C2ST)~\citep{friedman04multivariate, lopez-paz2017revisiting}, where the test accuracy of a random forest (RF) trained with labeled data (real data labeled with 1 and synthetic data labeled with 0) is considered as our test statistic. Less MMD means more similarity between distributions, and an RF accuracy close to chance level ($0.5$) indicates that the synthetic samples resemble the observational distribution.
Second, we use
mean absolute error ($\text{MAE}_\text{CF}$)
to measure the counterfactual
performance.

\subsection{Sensitivity analysis}

We have designed an experiment to test how sensible it is to adopt an additive noise model in a real case. Although in a real dataset we do not have access to the counterfactual error and MMD interventional, we can always measure the fit with the observational distribution and the independence between the extracted exogenous variables. That serves to falsify the additivity of the true DGP.

We generated 6000 samples of the following DGP and used the half for validation. 
\[
\begin{tikzpicture}[>=stealth, baseline=(m-2-1.base)]
\matrix (m) [matrix of math nodes,
             row sep=1.5pt,
             column sep=0.05cm,
             nodes={anchor=west},
             nodes in empty cells] {
  \covariatei_1  & = \exogenousi_1 \\
  \covariatei_2 &  =\frac{\covariatei_1}{2} + .1\covariatei_1^2 + .6\exogenousi_2\\
  \covariatei_3 & = .4\covariatei_1 + .1\covariatei_1^2 + .7\covariatei_2
  - \frac{\covariatei_2^2}{2} + \frac{\covariatei_2^3}{8}
  - \frac{\alpha \covariatei_2 \tanh(\exogenousi_3+1)}{2}
  + \frac{\exogenousi_3}{2} \\
};

\draw[->] (m-1-1) -- (m-2-1);
\draw[->] (m-2-1) -- (m-3-1);
\draw[->, bend right=40] (m-1-1) to (m-3-1);
\end{tikzpicture}
\]

\Cref{fig:sensitivity} presents the variation of the metrics used to falsify additivity in the true data-generating process. The charts present how the matching metrics of both observational and interventional distribution vary (MMD obs and MMD int) with $\alpha$ and how the error in counterfactual estimation varies and how the independence between $\exogenousi_3$ and $\parents{3}$ (HSIC) and between $\exogenousi_3$, $\exogenousi_2$, $\exogenousi_1$ (dHSIC) vary with $\alpha$. The value of $\alpha\in [0,1]$ determines if the DGP is additive ($\alpha=0$) or not ($\alpha>0$).

This experiment illustrates how we can reject the use of \ours, since a more flexible causal model achieves better observational matching and independence of the residuals when the noise is not additive, and this is something that can be tested based on observational data only.

\subsection{Performance on synthetic datasets}

We benchmark our approach to state-of-the-art methods on 11 synthetic graphs obtained from~\citet{geffner2022deep} under different data regimes. We conducted two different experiments, depending on whether all covariates are continuous or if there are discrete variables. All variables are standardized to obtain comparable counterfactual errors across variables. All experiments use the same amount of data for training and testing.

\subsubsection{Experiment on continuous data}

We compare the performance of our proposed approaches (\ours with KAN and KAAM) to the following methods: DBCM ~\citep{chao2024dcm} (diffusion-based), \methodname{causal flows}~\citep{javaloy2023causal}  (from now CFlow)  and ANM~\citep{hoyer2009anm} (Additive Noise Model, which selects the best performing method among several classical machine learning tools using regression metrics~\citep{dowhy_gcm}). Both DBCM and CFlow are universal density approximators, although they need, in principle, larger sample sizes to train~\citep{khemakhem2021caf}.

Once the best hyperparameters (details in
App. B in supplementary material) are obtained, we proceed to train our final model and validate the results. The validation is done along three main axes: observational, interventional and counterfactual metrics. For \emph{observational metrics}, we compare the RF accuracy and the MMD between the generated distribution and the validation samples. To measure \emph{interventional prediction},  we generate six different interventions over variables $\covariatei_1$ and $\covariatei_2$ (always present in the graph) at standardized values $\covariatei_\indexcovariate=\{-1, 0, 1\}$. The corresponding interventional distributions were evaluated using RF accuracy and MMD with the validation samples. Finally, for \emph{counterfactual inference}, we use the validation samples to compute the counterfactual values of each sample using the same interventions defined in the previous point. The counterfactual error is evaluated using MAE ($\text{MAE}_\text{CF}$), compared with the values of the true SCM.

\subsubsection{Experiment on mixed data}
\begin{table*}[ht]
\scriptsize
\setlength{\tabcolsep}{3pt}
\centering
\begin{tabular}{lcccccc|cccccc}
\hline
& \multicolumn{6}{c}{Additive} & \multicolumn{6}{c}{Nonadditive} \\
\cmidrule(lr){2-7} \cmidrule(lr){8-13}
Model 
& $\text{RF}_\text{acc}^\text{obs}$ 
& $\text{MMD}_\text{obs}$ 
& $\text{RF}_\text{acc}^\text{int}$ 
& $\text{MMD}_\text{int}$ 
& $\text{MAE}_\text{CF}$ 
& p-value
& $\text{RF}_\text{acc}^\text{obs}$ 
& $\text{MMD}_\text{obs}$ 
& $\text{RF}_\text{acc}^\text{int}$ 
& $\text{MMD}_\text{int}$ 
& $\text{MAE}_\text{CF}$ 
& p-value \\

& & $\times 10^{-3}$ & & $\times 10^{-3}$ & $\times 10^{-2}$ &
& & $\times 10^{-3}$ & & $\times 10^{-3}$ & $\times 10^{-2}$ & \\
\hline

KAN 
& $0.54_{0.02}$ & $1.36_{0.41}$ & $0.54_{0.02}$ & $1.95_{0.67}$ & $3.50_{1.40}$ & $0.180^*$
& $0.67_{0.13}$ & $5.27_{9.20}$ & $0.67_{0.14}$ & $6.09_{8.75}$ & $3.42_{1.54}$ & $<1\text{e-}3$ \\

KAAM 
& $0.54_{0.01}$ & $1.35_{0.40}$ & $0.55_{0.01}$ & $1.96_{0.74}$ & $3.59_{1.35}$ & $0.180^*$
& $0.68_{0.14}$ & $5.97_{9.60}$ & $0.68_{0.14}$ & $6.61_{9.12}$ & $3.56_{1.68}$ & $<1\text{e-}3$ \\

ANM 
& $\mathbf{0.53_{0.01}}$ & $\mathbf{1.12_{0.28}}$ & $\mathbf{0.53_{0.01}}$ & $\mathbf{1.30_{0.21}}$ & $\mathbf{1.28_{0.43}}$ & $\underline{0.500}^*$
& $0.61_{0.09}$ & $1.71_{0.83}$ & $0.61_{0.09}$ & $2.31_{1.46}$ & $1.76_{1.13}$ & $0.007$ \\

DBCM 
& $0.56_{0.01}$ & $2.78_{0.73}$ & $0.56_{0.01}$ & $2.38_{0.54}$ & $3.03_{0.47}$ & $<1\text{e-}3$
& $0.59_{0.04}$ & $2.61_{0.78}$ & $0.60_{0.04}$ & $2.34_{0.60}$ & $2.67_{0.78}$ & $0.157^*$ \\

CFlow 
& $0.53_{0.01}$ & $1.42_{0.26}$ & $0.53_{0.01}$ & $1.53_{0.29}$ & $1.69_{0.52}$ & $0.396^*$
& $\mathbf{0.55_{0.03}}$ & $\mathbf{1.42_{0.35}}$ & $\mathbf{0.59_{0.08}}$ & $\mathbf{1.84_{0.96}}$ & $\mathbf{1.84_{0.99}}$ & $\underline{0.500}^*$ \\

\hline
\end{tabular}

\caption{
Results on the Sachs dataset for additive and nonadditive noise settings. Metrics report averages over runs ($\text{mean}_\text{std}$). Lower is better for all metrics. The last column of each block reports p-values from post-hoc Wilcoxon tests aggregated over all metrics. The Friedman test p-value is always $<1\text{e-}3$, indicating significant differences among methods overall. Underline denotes the best method. $^*$ indicates no statistically significant difference ($p>0.05$).
}
\label{tab:sachs_results_unified}

\end{table*}

For the discrete case, we proceed in a similar fashion, with some key differences. First, we adapt the validation of models designed for continuous variables (CFlow and DBCM) by \itemi adding a small scale noise to the training data (standard deviation $0.01$); and \itemii post-processing the data generated for the discrete variables, by rounding. Second, we add \oursmix to the evaluation. Third, we discretize one variable, $\covariatei_3  \in \{0, 1, 2\}$, in the true DGP of~\citet{geffner2022deep} by sampling a ternary variable with probability vectors: $[0.8, 0.1, 0.1]$ if $\covariatei_3 <-1$; $[0.1, 0.8, 0.1]$ if $-1<\covariatei_3 <1$; and $[0.1, 0.1, 0.8]$ if $\covariatei_3 >1$. Finally, we do not evaluate counterfactual prediction, since point-valued counterfactuals are not identifiable.

\begin{figure*}[ht]
    \centering
    \includegraphics[width=0.70\linewidth]{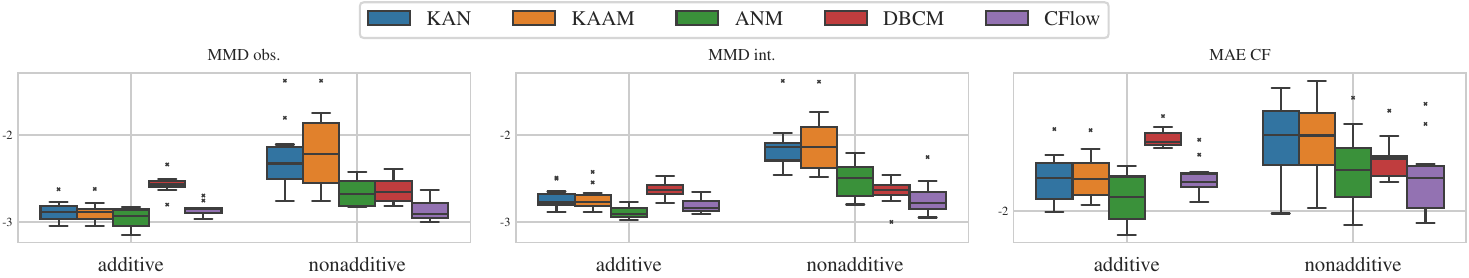}
    \caption{Boxplots of $\log_{10}$ of observational MMD, interventional MMD and CF MAE on Sachs' semisynthetic dataset. While in the additive setting, all models achieve similar performance, \ours worsen significantly in the nonadditive setting.}
    \label{fig:sachs_results}
\end{figure*}

\textbf{Results}.
The results for 1000 samples can be observed in \cref{tab:results_unified}, showing close performance of all the models, and the p-values comparing the models in all the metrics together,  using the procedure of~\citet{demvsar2006statistical}. P-values show that \itemi in the continuous case, ANM is the best model, while the two variants of \ours are not rejectable as the best models for type I error probability: $\alpha=0.05$, and \itemii in the discrete case,  \oursmix are the best models across all metrics and realizations, supporting the proposal of using mixed models with discrete data.  Note that CFlow and DBCM perform slightly worse because the number of samples is low and the true DGPs of~\citet{geffner2022deep} lie within the ANM class.

We also evaluate the impact of the number of samples on the metrics. Hence, for each model and dataset, we compute all metrics using $1000$, $100$ and $10$ training (and test) samples. When the number of samples decreases, the metrics worsen, as expected: this behavior can be observed in \cref{fig:synthetic_results}.

\subsection{Performance in semi-synthetic dataset}

\begin{figure*}[ht]
    \centering
    \includegraphics[width=0.9\linewidth]{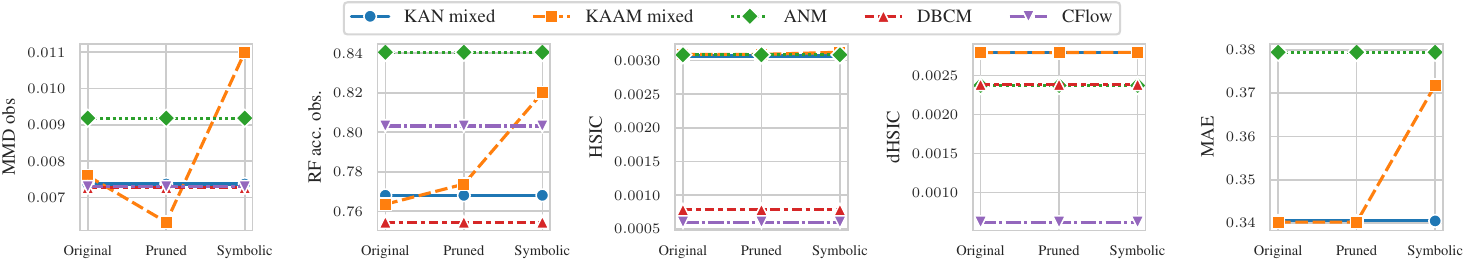}
    \caption{Observational metrics on the cardiovascular dataset as a function of the simplification stage. MAE denotes predictive error when estimating each node from its parents. DBCM and CFlow cannot compute conditional samples and, therefore, MAE is not reported. HSIC and dHSIC are evaluated only for continuous variables.}
    \label{fig:cardio_results}
\end{figure*}

Following other works~\citep{chao2024dcm, almodovar2025decaflow}, we also validate our approach on a more complex semi-synthetic dataset, created following the process procedure by~\citet{chao2024dcm}, from the original dataset of~\citet{sachs2005causal}. The dataset contains 11 variables associated with the causal graph of \cref{fig:sachs_graph}. Root nodes are sampled from the original dataset, while the other variables are generated using random equations.
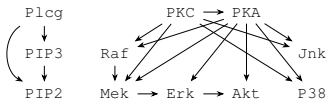
\begin{figure}[H]
    \centering
         \begin{tikzcd}[
        scale=0.2, column sep=small, row sep=small]
        \Plcg \arrow[d] \arrow[dd, bend right=60]  & &
        \PKC \arrow[r]
                         \arrow[dl]
                         \arrow[ddl]
                         \arrow[drr]
                         \arrow[ddrr]
        & \PKA \arrow[dll]
                         \arrow[ddll]
                         \arrow[ddl]
                         \arrow[dd]
                         \arrow[dr]
                         \arrow[ddr] & \\
        \pipthree \arrow[d]& \Raf \arrow[d] & & & \Jnk \\
        \piptwo & \Mek \arrow[r] 
         & \Erk \arrow[r]  & \Akt  & \Pte
    \end{tikzcd}
    \caption{Causal graph associated with Sachs' dataset.}
    \label{fig:sachs_graph}
\end{figure}
We generate $1000$ training samples and also $1000$ validation samples in both additive and nonadditive settings. The results using the same evaluation method from the previous experiment can be seen in \cref{tab:sachs_results_unified}. In the additive setting, ANM achieves the best results in all metrics, while \ours and CFlows (in particular causal masked autoregressive flows, \methodname{CausalMAF}, which is also an additive model) get comparable results, while DBCM is not competitive with this number of samples. On the other hand, in the nonadditive settings, \methodname{causal flows} (in particular, causal neural spline flows, \methodname{CausalNSF}, which is more flexible) achieves the best results, with DBCM presenting a competitive performance. These findings can also be observed in \cref{fig:sachs_results}.

\subsection{Use case: cardiovascular data}

We illustrate the interpretability capabilities of our approach on a real-world cardiovascular dataset whose goal is to estimate the probability of a \emph{major acute cardiovascular event} (MAC). The dataset, introduced by~\citet{KYRIACOU2023606}, contains continuous and binary clinical covariates together with a causal graph that we assume to be correct and is shown in \cref{fig:cardio_graph}.

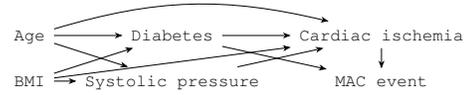
\begin{figure}[H]
    \centering
        \begin{tikzcd}[
        scale=0.2,
        column sep=small, row sep=small]
        \age \arrow[r] \arrow[dr] \arrow[rr, bend left=15] & \diabetes  \arrow[r] \arrow[dr] & \ischemia \arrow[d] \\
        \obesity \arrow[r] \arrow[ur] \arrow[urr] & \hypertension \arrow[ur] & \mac
    \end{tikzcd}
    \caption{Causal graph of cardio dataset. The original graph contained \textit{obesity} and \textit{hypertension}; we instead use their continuous counterparts \methodname{BMI} and \methodname{systolic pressure}.}
    \label{fig:cardio_graph}
\end{figure}

\begin{figure}[ht]
    \centering
    \includegraphics[width=0.7\linewidth, trim={1cm 0 3cm 0 }, clip]{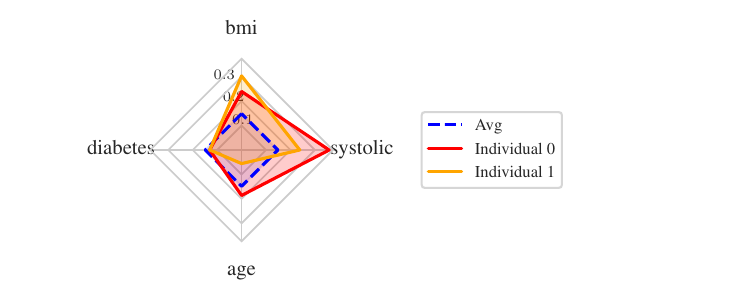}
     \includegraphics[width=0.95\linewidth]{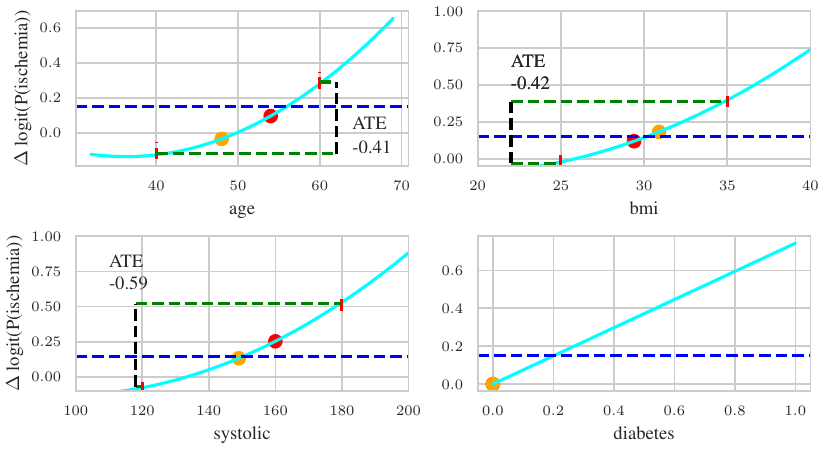}
    \caption{Interpretability analysis of the ischemia mechanism. 
(Top) PRPs showing the contribution of each parent variable to the logit for two individuals. 
(Bottom) PDPs representing the contribution $\Delta$ of each variable to the logit. The curves also allow direct computation of average treatment effects (ATE) corresponding to hypothetical interventions on each variable.}
    \label{fig:ischemia}
\end{figure}

Since true interventional distributions or counterfactual outcomes are not available, the evaluation must rely on observational diagnostics. Following the validation pipeline described in \cref{sec:method}, we train \oursmix using KAN and KAAM subnetworks and compare it with the baselines. In addition, we apply the simplification pipeline (pruning and symbolic substitution) and evaluate how these transformations affect the observational fit.

The results are summarized in \cref{fig:cardio_results}. Pruning preserves the observational performance, whereas symbolic substitution slightly degrades the metrics, illustrating the expected trade-off between predictive fidelity and interpretability. Nevertheless, \oursmix achieves competitive observational fit, obtaining the best RF accuracy and MMD among the compared models, although independence diagnostics are less favorable. This analysis supports the use of the simplified models when interpretability is a primary objective.

The key advantage of the proposed model is that the structural mechanisms can be explicitly inspected. For example, after pruning and symbolic approximation, the probability of cardiac ischemia conditioned on its parents can be expressed as
\begin{align*}
\small
\Pr(\text{ischemia}) & =
\sigma(0.055\,\overline{\text{age}}^2
+0.173\,\overline{\text{age}} \\
& +0.031\,\overline{\text{BMI}}^2
+0.114\,\overline{\text{BMI}}
\\
&
+0.046\,\overline{\text{systolic}}^2
+0.155\,\overline{\text{systolic}}\\
& +0.743\,\text{diabetes}
-1.871),
\end{align*}
where $\sigma$ represents the sigmoid function and the overline variables are standardized. 
This expression provides a transparent representation of the learned causal mechanism that can be directly audited by domain experts. For instance, the logit of the ischemia probability increases quadratically with both age and systolic pressure, a pattern that is consistent with clinical intuition: cardiovascular risk is known to grow nonlinearly with aging and with elevated blood pressure. Similarly, BMI contributes through both linear and quadratic terms, capturing increasing risk at higher levels of body mass.

Beyond the closed-form expression, the interpretability tools allow for visualizing how each variable contributes to the predicted outcome. \cref{fig:ischemia} illustrates this through two complementary visualizations. The probability radar plots (PRP) decompose the contribution of each parent variable to the logit for two randomly selected individuals. In the example shown, Individual~0 exhibits higher ischemia risk due to greater age and systolic pressure, whereas Individual~1 has elevated BMI, which also increases the predicted probability.

The probability dependence plots (PDPs) provide a complementary population-level view of the learned mechanisms. These plots display how each parent variable contributes to the logit of the ischemia probability while marginalizing over the remaining variables. Because the structural equation is explicit, these curves also allow straightforward estimation of average treatment effects (ATE) corresponding to hypothetical interventions on each variable.

For instance, the PDPs in \cref{fig:ischemia} show how changes in BMI and systolic pressure modify the logit of the ischemia probability. From these curves, the ATE of reducing BMI from 35 to 20 is approximately $-0.42$ in the logit scale, while lowering systolic pressure from 180 to 120 yields an ATE of approximately $-0.59$. These quantities correspond to population-level intervention effects estimated directly from the learned structural equation.

Importantly, these effects are not obtained from a black-box predictor but from an explicit functional representation of the mechanism. Consequently, practitioners can both visualize the nonlinear contribution of each variable and compute interpretable causal summaries such as ATEs directly from the model.

We provide the code of all the experiments for reproducibility in  \href{\codelink}{\codeshort}. We also include an online interactive interface, in which the causal effects of every variable in the Sachs' dataset and in the usecase can be explored and audited, in \href{\interfacelink}{\interfaceshort}.
\section{Concluding remarks}
\label{sec:conclusion}

This work investigated the integration of Kolmogorov-Arnold networks within causal generative modeling for tabular data. We proposed \ours, a causal generative model that preserves the semantics of structural causal models while enabling direct inspection of the learned structural mechanisms. \ours is a query-agnostic model capable of answering observational, interventional, and counterfactual queries, while supporting symbolic simplification, pruning, and functional visualization.

Empirically, the results presented in \cref{sec:experiments} demonstrate that, under additive data-generating processes, \ours achieves performance comparable to optimal additive noise models and competitive with state-of-the-art deep causal generative models in terms of observational fit, interventional distribution matching, and counterfactual accuracy. In mixed continuous--discrete settings, the proposed extension maintains competitive performance while preserving interpretability of the structural equations. The sensitivity analysis further illustrates that model misspecification can be detected using independence tests on inferred exogenous variables and distributional metrics, providing practical guidance for deployment under observational data alone. In semi-synthetic experiments, \ours performs strongly when the additive assumption is satisfied, while more flexible deep causal models outperform it in strongly nonadditive regimes. This trade-off highlights the central design principle of our approach: favoring structural interpretability and auditability when the data-generating process is compatible with additive mechanisms.
Finally, in the real-world cardiovascular use case, we demonstrated that the model can achieve competitive observational performance while enabling the extraction of simplified structural equations and interpretable visualizations of causal effects. These results support the feasibility of combining expressive generative modeling with functional transparency, a property that is particularly relevant for high-stakes applications where auditing and mechanistic understanding are required.

Despite its advantages, the proposed approach presents several limitations. First, the model relies on structural mechanisms that admit relatively low-complexity functional decompositions. When the true data-generating process exhibits strongly nonadditive, highly entangled, or heteroscedastic interactions, more flexible deep causal generative models can achieve superior density approximation and counterfactual accuracy, as observed in \cref{sec:experiments}. In such regimes, enforcing structured functional representations may lead to approximation bias. 
Second, interpretability is inherently tied to model simplification. Pruning and symbolic approximation improve transparency but may degrade predictive fidelity, introducing a trade-off between auditability and accuracy. While our validation pipeline mitigates this risk through independence and distributional diagnostics, it cannot guarantee that the learned mechanisms coincide with the true structural equations, particularly under model misspecification. 
Third, the methodology assumes causal sufficiency and a correctly specified causal graph. Violations of these assumptions, such as hidden confounding or graph misspecification, can invalidate the generative interpretation and the associated counterfactual estimates. Finally, although the approach scales to moderate-dimensional tabular settings, extracting symbolic expressions and visual diagnostics becomes increasingly challenging as graph size and mechanism complexity grow. Addressing these limitations requires balancing interpretability, expressivity, and scalability in future work.

Several directions for future research arise from this work. First, one may extend the architecture to increase density approximation flexibility while preserving interpretability. In particular, KAN-based normalizing flow constructions could allow more expressive transformations of the exogenous variables, including
neural spline flows~\citep{durkan2019neural} or unconstrained monotonic neural networks~\citep{wehenkel2019unconstrained}. Such extensions would require new interpretability tools to ensure that increased flexibility does not compromise auditability.
Second, alternative functional bases beyond spline parameterizations may be explored, such as Fourier expansions~\citep{li2025kolmogorov}, which could provide improved inductive biases in periodic or oscillatory mechanisms.
Third, assessing partial well-specification represents an important practical direction. Instead of requiring global model adequacy, one may restrict attention to subsets of intervention targets for which additive structure is sufficiently accurate, as suggested in recent work on partial model validation~\citep{schultheiss24assesing}. Another promising avenue is regression by dependence minimization~\citep{mooij2009regression}, enforcing independence between inferred exogenous variables and their parents to strengthen structural validity.

Finally, extending the methodology beyond causal sufficiency remains an open challenge. Incorporating proxy-based adjustments~\citep{almodovar2025decaflow} or leveraging graph transformations that yield equivalent identifiable structures~\citep{parafita2022estimand} would broaden applicability to settings with hidden confounding.

In summary, this work advances interpretable causal generative modeling by demonstrating that structured, additive neural architectures can provide competitive performance while enabling direct inspection of structural equations. We hope this contribution encourages further research at the intersection of causal modeling, generative learning, and functional interpretability.

\bibliography{ref}
\bibliographystyle{myplainnat}

\appendix

\subsection{Proofs}
\label[app]{app:sec:proofs}

\begin{proof}[Proof of \cref{prop:interventional_mixed}]
    
    \itemii Intervening in $\covariatei_\indexfour \in \parents{\indexcovariate}$ , $\doop(\covariatei_\indexfour = \sampleintervention)$, the mechanism is $\covariatei = \pred{\func[\parameters]}(\sampleintervention, \parents{\indexcovariate}_{-\indexfour}, \exogenousi_\indexcovariate)$. Conditioning on $\sampleparents{\indexcovariate}_{-\indexfour}$, and having that $\exogenousi_\indexcovariate \indep \parents{\indexcovariate}$:

    \begin{align*}
    \pdf_\parameters(\giventhat{\samplecovi_\indexcovariate}{\doop(\covariatei_\indexfour = \sampleintervention), \sampleparents{\indexcovariate}}) 
    &= \distribution(\pred{\func[\parameters]}(\sampleintervention, \sampleparents{\indexcovariate}, \exogenousi_\indexcovariate) = \samplecovi_\indexcovariate) \\
    &= \pdf_\parameters(\giventhat{\samplecovi}{\sampleintervention, \sampleparents{\indexcovariate}}).
    \end{align*}

    For descendants of $\covariatei_\indexcovariate$, with downstream mechanisms fixed, the interventional law factorizes as a mixture over $\covariatei_\indexcovariate$ with weights $\pdf_\parameters(\giventhat{\samplecovi_\indexcovariate}{\doop(\sampleintervention), \sampleparents{\indexcovariate}_{-\indexfour}})$. Convergence in these weights implies convergence in distribution of the induced descendants.
\end{proof}

\begin{proof}[Proof of \cref{prop:counterfactual_mixed}]
    Since $\descendants{\intervention}\in \gC$ For each variable $\indexcovariate \in \descendants{\intervention}$, every parent of \indexcovariate lies \itemi in \intervention, \itemii in \descendants{\intervention} or \itemiii outside \descendants{\intervention}. In the abduction-action-prediction process, given a factual point, \factual{\samplecov}, \itemi is fixed by intervention, \itemii is computed earlier in the topological evaluation and \itemiii keep their factual values. Thus, the inputs to \pred{\func[\indexcovariate]} are uniquely determined (in addition to $\pred{\sampleexogenousi}_\indexcovariate = \pred{\func[\indexcovariate]}^{-1}(\parents{\indexcovariate}, \factual{\samplecovi}_\indexcovariate)$), which gives a unique value of the counterfactual $\cfactual{\samplecovi}_{\indexcovariate, \doop(\samplecovi_\intervention)}$.
\end{proof}
\subsection{Operation of \ours}
\label[app]{app:sec:operation}

Details of training (\cref{alg:training}),  the do-operator implementation (\cref{alg:interventional} and \cref{alg:counterfactual}) and assumptions falsification (\cref{alg:falsify}). Without loss of generality, assume that \covariates are in causal order. \regularization is the regularization term of KANs.

\begin{algorithm}[ht]
\scriptsize
\caption{Training}
\label{alg:training}
\begin{algorithmic}[1]
\REQUIRE Dataset $\dataset$
\ENSURE $\{\parameters_\indexcovariate^\star\}_{\indexcovariate=1}^{\covsize}$, $\{\pred{\pdf}(\exogenousi_\indexcovariate)\}_{\indexcovariate=1}^{\covsize}$
\FOR{$\indexcovariate = 1$ to $\covsize$}
    \STATE $\lambda^\star_\indexcovariate = \arg\min_{\lambda} \left[ \min_{\parameters} \loss \right]$
    \STATE $\parameters_\indexcovariate^\star = \arg\min_{\parameters \in \Theta(\lambda^\star_\indexcovariate)} 
    \text{MSE}\!\left(\covariatei_\indexcovariate, 
    \text{KAN}^{(\indexcovariate)}_\parameters(\parents{\covariatei_\indexcovariate})\right) + \regularization$
    \STATE $\pred{\exogenous_\indexcovariate} \leftarrow \text{ObtainResiduals}(\dataset, \func[\parameters^\star])$ $//$ \cref{alg:residuals} 
    \STATE $\pred{\pdf}(\exogenousi_\indexcovariate) \leftarrow \text{EstimateDistribution}(\pred{\exogenous_\indexcovariate})$
\ENDFOR
\STATE \textbf{return} $\{\parameters_\indexcovariate^\star\}_{\indexcovariate=1}^{\covsize}$, $\{\pred{\pdf}(\exogenousi_\indexcovariate)\}_{\indexcovariate=1}^{\covsize}$
\end{algorithmic}
\end{algorithm}

\begin{algorithm}[ht]
\scriptsize
\caption{Obtain Residuals}
\label{alg:residuals}
\begin{algorithmic}[1]
\REQUIRE Dataset $\dataset$, functions $\func[\parameters] = \{\pred{\func}\}_{\indexcovariate=1}^{\covsize}$
\ENSURE $\{\{\pred{\sampleexogenousi}_\indexcovariate^{(\indexsample)}\}_{\indexsample=1}^{\samplesize}\}_{\indexcovariate=1}^{\covsize}$
\FOR{$\indexsample = 1$ to $\samplesize$}
    \FOR{$\indexcovariate = \covsize$ down to $1$}
        \STATE $\pred{\sampleexogenousi}_\indexcovariate^{(\indexsample)} \leftarrow 
        \samplecovi_\indexcovariate^{(\indexsample)} -
        \pred{\func}(\parents{\indexcovariate}^{(\indexsample)})$
    \ENDFOR
\ENDFOR
\STATE \textbf{return} $\{\{\pred{\sampleexogenousi}_\indexcovariate^{(\indexsample)}\}\}_{\indexcovariate,\indexsample}$
\end{algorithmic}
\end{algorithm}

\begin{algorithm}[ht]
\scriptsize
\caption{Falsify Assumptions}
\label{alg:falsify}
\begin{algorithmic}[1]
\REQUIRE Dataset $\dataset$, functions $\func[\parameters]$, intervention set $\intervention$
\ENSURE $\{\text{p-value}_\indexcovariate\}_{\indexcovariate=1}^{\covsize}$, MMD, $\text{RF}_{\text{acc}}$
\STATE $\{\{\pred{\sampleexogenousi}_\indexcovariate^{(\indexsample)}\}\} \leftarrow \text{ObtainResiduals}(\dataset, \func[\parameters])$$//$ \cref{alg:residuals} 
\FOR{$\indexcovariate = 1$ to $\covsize$}
    \STATE $\text{stat}_\indexcovariate \leftarrow 
    \text{HSIC}\!\left(\{\pred{\sampleexogenousi}_\indexcovariate^{(\indexsample)}\}_{\indexsample=1}^{\samplesize},
    \{\parents{\indexcovariate}^{(\indexsample)}\}_{\indexsample=1}^{\samplesize}\right)$
\ENDFOR
\FOR{$\indexcovariate \in \{1,\dots,\covsize\}\setminus\intervention$}
    \STATE Sample $\{\sampleexogenousi_\indexcovariate^{(\indexthree)}\}_{\indexthree=1}^{M} \sim \pred{\pdf}(\exogenousi_\indexcovariate)$
    \STATE $\samplecovi_\indexcovariate^{(\indexthree)} \leftarrow 
    \pred{\func}(\parents{\indexcovariate}^{(\indexthree)}) +
    \sampleexogenousi_\indexcovariate^{(\indexthree)}$
\ENDFOR
\STATE $\text{MMD} \leftarrow \text{MMD}\left(\dataset, \{\{\samplecovi_\indexcovariate^{(\indexthree)}\}_{\indexcovariate=1}^{\covsize}\}_{\indexthree=1}^{M}\right)$
\STATE Train RF to discriminate real vs generated samples and compute $\text{RF}_{\text{acc}}$
\STATE \textbf{return} $\{\text{stat}_\indexcovariate\}$, MMD, $\text{RF}_{\text{acc}}$
\end{algorithmic}
\end{algorithm}

\begin{algorithm}[ht]
\scriptsize
\caption{Interventional Sampling}
\label{alg:interventional}
\begin{algorithmic}[1]
\REQUIRE Dataset $\dataset$ in causal order, functions $\func[\parameters]$, intervention values $\{\sampleintervention_\indexcovariate\}_{\indexcovariate \in \intervention}$, $M$
\ENSURE $\{\{\samplecovi_\indexcovariate^{(\indexthree)}\}_{\indexcovariate=1}^{\covsize}\}_{\indexthree=1}^{M}$
\FOR{$\indexcovariate \in \intervention$}
    \STATE $\samplecovi_\indexcovariate \leftarrow \sampleintervention_\indexcovariate$
\ENDFOR
\FOR{$\indexcovariate \in \{1,\dots,\covsize\}\setminus\intervention$}
    \STATE Sample $\{\sampleexogenousi_\indexcovariate^{(\indexthree)}\}_{\indexthree=1}^{M} \sim \pred{\pdf}(\exogenousi_\indexcovariate)$
    \STATE $\samplecovi_\indexcovariate^{(\indexthree)} \leftarrow 
    \pred{\func}(\parents{\indexcovariate}^{(\indexthree)}) +
    \sampleexogenousi_\indexcovariate^{(\indexthree)}$
\ENDFOR
\STATE \textbf{return} generated interventional samples
\end{algorithmic}
\end{algorithm}

\begin{algorithm}[ht]
\scriptsize
\caption{Counterfactual Inference}
\label{alg:counterfactual}
\begin{algorithmic}[1]
\REQUIRE Factual covariates $\factual{\covariates}$ in causal order, functions $\func[\parameters]$, intervention set $\intervention$
\ENSURE Counterfactual covariates $\{\samplecovi_\indexcovariate\}_{\indexcovariate=1}^{\covsize}$
\FOR{$\indexcovariate = \covsize$ down to $1$}
    \STATE $\factual{\pred{\sampleexogenousi}_\indexcovariate} \leftarrow 
    \factual{\covariatei}_\indexcovariate -
    \func[\parameters](\factual{\parents{\indexcovariate}})$
\ENDFOR
\FOR{$\indexcovariate \in \intervention$}
    \STATE $\samplecovi_\indexcovariate \leftarrow \sampleintervention_\indexcovariate$
\ENDFOR
\FOR{$\indexcovariate \in \{1,\dots,\covsize\}\setminus\intervention$}
    \STATE $\samplecovi_\indexcovariate \leftarrow 
    \func[\parameters](\cfactual{\parents{\indexcovariate}}) +
    \factual{\pred{\sampleexogenousi}_\indexcovariate}$
\ENDFOR
\STATE \textbf{return} $\{\samplecovi_\indexcovariate\}_{\indexcovariate=1}^{\covsize}$
\end{algorithmic}
\end{algorithm}

\subsection{Hyperparameter search}
\label[app]{app:sec:hyperparameters}

The experiments were conducted after a hyperparameter search. We include here the specific grid that we used to increase reproducibility.

\begin{itemize}
    \item KAN: Hidden dimension size ($0,5$); Grid size ($1, 5, 10$); Learning rate ($1e-3, 1e-3, 1e-4$), L1 regressor strength ($0.1, 0.01, 0.001$), Use of multiplicative nodes (yes or no, binary option). 
    \item DBCM: Number of training epochs ($100, 200, 500$); Learning rate ($1e-2, 1e-3, 1e-4$); Batch size ($32, 64, 128$), Hidden dimension size ($32, 64, 128$).
    \item CFlow: Flow type (CausalNSF or CausalMAF); Hidden dimension sizes ($(32, 32)$, $(64, 64)$, $(64, 64, 64)$, $(128, 128)$); Learning rates ($1e-2, 1e-3, 1e-4, 1e-5$); Scheduler (None or plateau); Number of bins ($4, 8, 16$).
\end{itemize}

We then select, for each child, the best hyperparameter combination based on the RF accuracy between the generated and the observational data for the child and all covariates in the graph. In order to find the best KAAM hyperparameters, we reuse the results of the KAN grid search, and keep the best hyperparameters per child among those who had no hidden dimension (hidden dimension size is $0$) and do not make use of multiplicative nodes. 

For DBCM and CFlow,
we select based on RF accuracy between the generated and the observational data.
Finally, for ANM, the library used does an automatic search to find the best model in terms of regression performance. 

\subsection{Benchmark graphs description \label{app:sec:equations}}

We now introduce the $11$ synthetic graphs used for benchmarking in \cref{sec:experiments} of this work. The graphs are based on \citep{geffner2022deep}, and are designed to provide a combination of linear and non-linear equations. All exogenousi variables $\exogenousi_{\indexone}$ follow a standard Gaussian distribution $\normal(0,1)$ and we standardize all variables before generating the next one.

\textbf{3 Chain Linear}

\begin{equation*}
    \begin{cases}
        \covariatei_1 &= \func[1](\exogenousi_1) = \exogenousi_1 \\
        \covariatei_2 &= \func[2](\covariatei_1, \exogenousi_2) = 10\cdot \covariatei_1 - \exogenousi_2 \\
        \covariatei_3 &= \func[3](\covariatei_2, \exogenousi_3) = 0.25 \cdot \covariatei_2 + 2\cdot \exogenousi_3
    \end{cases}
\end{equation*}

\textbf{3 Chain non Linear}

\begin{equation*}
    \begin{cases}
        \covariatei_1 &= \func[1](\exogenousi_1) = \exogenousi_1 \\
        \covariatei_2 &= \func[2](\covariatei_1, \exogenousi_2) = e^{0.5 \cdot \covariatei_1} +0.25 \cdot \exogenousi_2 \\
        \covariatei_3 &= \func[3](\covariatei_2, \exogenousi_3) = \frac{\left(\covariatei_2 - 5 \right)^3}{15} + \exogenousi_3
    \end{cases}
\end{equation*}

\textbf{4 Chain Linear}

\begin{equation*}
    \begin{cases}
        \covariatei_1 &= \func[1](\exogenousi_1) = \exogenousi_1 \\
        \covariatei_2 &= \func[2](\covariatei_1, \exogenousi_2) = 5 \cdot \covariatei_1 - \exogenousi_2 \\
        \covariatei_3 &= \func[3](\covariatei_2, \exogenousi_3) = -0.5 \cdot \covariatei_2 - 1.5 \cdot \exogenousi_3 \\
        \covariatei_4 &= \func[4](\covariatei_3, \exogenousi_4) = \covariatei_3 + \exogenousi_4
    \end{cases}
\end{equation*}

\textbf{5 Chain Linear}

\begin{equation*}
    \begin{cases}
        \covariatei_1 &= \func[1](\exogenousi_1) = \exogenousi_1 \\
        \covariatei_2 &= \func[2](\covariatei_1, \exogenousi_2) = 10 \cdot \covariatei_1 - \exogenousi_2 \\
        \covariatei_3 &= \func[3](\covariatei_2, \exogenousi_3) = 0.25 \cdot \covariatei_2 + 2 \cdot \exogenousi_3 \\
        \covariatei_4 &= \func[4](\covariatei_3, \exogenousi_4) = \covariatei_3 + \exogenousi_4 \\
        \covariatei_5 &= \func[5](\covariatei_4, \exogenousi_5) = -\covariatei_4 + \exogenousi_5
    \end{cases}
\end{equation*}

\textbf{Collider Linear}

\begin{equation*}
    \begin{cases}
        \covariatei_1 &= \func[1] (\exogenousi_1) = \exogenousi_1 \\
        \covariatei_2 &= \func [2](\exogenousi_2) = 2 - \exogenousi_2 \\
        \covariatei_3 &= \func[3] (\covariatei_1, \covariatei_2, \exogenousi_3) = 0.25 \cdot \covariatei_2 - 0.5 \cdot \covariatei_1 + 0.5\cdot \exogenousi_3
    \end{cases}
\end{equation*}

\textbf{Fork Linear}

\begin{equation*}
    \begin{cases}
        \covariatei_1 &= \func[1](\exogenousi_1) = \exogenousi_1 \\
        \covariatei_2 &= \func[2](\exogenousi_2) = 2 - \exogenousi_2 \\
        \covariatei_3 &= \func[3](\covariatei_1, \covariatei_2, \exogenousi_3) = 0.25 \cdot \covariatei_2 - 1.5 \cdot \covariatei_1 + 0.5\cdot \exogenousi_3 \\
        \covariatei_4 &= \func[4](\covariatei_3, \exogenousi_4)= \covariatei_3 + 0.25 \cdot \exogenousi_4
    \end{cases}
\end{equation*}

\textbf{Fork non Linear}

\begin{equation*}
    \begin{cases}
        \covariatei_1 &= \func[1](\exogenousi_1) = \exogenousi_1 \\
        \covariatei_2 &= \func[2](\exogenousi_2) = \exogenousi_2 \\
        \covariatei_3 &= \func[3](\covariatei_1, \covariatei_2, \exogenousi_3) = \frac{4}{1 + e^{-\covariatei_1 - \covariatei_2}} - \covariatei_2^2 + 0.5 \cdot \exogenousi_3 \\
        \covariatei_4 &= \func[4](\covariatei_3, \exogenousi_4)= \frac{20}{1+e^{0.5\cdot \covariatei_3^2 - \covariatei_3}} + \exogenousi_4
    \end{cases}
\end{equation*}

\textbf{Simpson non Linear}

\begin{equation*}
    \begin{cases}
        \covariatei_1 &= \func[1](\exogenousi_1) = \exogenousi_1 \\
        \covariatei_2 &= \func[2](\covariatei_1, \exogenousi_2) = \log \left(1 + e^{1-\covariatei_1} \right) + \sqrt{\frac{3}{20}} \exogenousi_2 \\
        \covariatei_3 &= \func[3](\covariatei_1, \covariatei_2, \exogenousi_3) = \tanh (2 \cdot \covariatei_2) + 1.5 \cdot \covariatei_1 - 1 + \tanh (\exogenousi_3) \\
        \covariatei_4 &= \func[4](\covariatei_3, \exogenousi_4)= \frac{\covariatei_3 - 4}{5} + 3+  \frac{\exogenousi_4}{\sqrt{10}}
    \end{cases}
\end{equation*}

\textbf{Simpson Symprod}

\begin{equation*}
    \begin{cases}
        \covariatei_1 &= \func[1](\exogenousi_1) = \exogenousi_1 \\
        \covariatei_2 &= \func[2](\covariatei_1, \exogenousi_2) = 2 \cdot \tanh (2 \cdot \covariatei_1) + \frac{\exogenousi_2}{\sqrt{10}} \\
        \covariatei_3 &= \func[3](\covariatei_1, \covariatei_2, \exogenousi_3) = 0.5 \cdot \covariatei_1 \cdot \covariatei_2 + \frac{\exogenousi_3}{\sqrt{2}} \\
        \covariatei_4 &= \func[4](\covariatei_1, \exogenousi_4)= \tanh (1.5 \cdot \covariatei_1) + \sqrt{\frac{3}{10}} \cdot \exogenousi_4
    \end{cases}
\end{equation*}

\textbf{Triangle Linear}

\begin{equation*}
    \begin{cases}
        \covariatei_1 &= \func[1](\exogenousi_1) = \exogenousi_1 \\
        \covariatei_2 &= \func[2](\covariatei_1, \exogenousi_2) = 10 \cdot \covariatei_1 - \exogenousi_2 \\
        \covariatei_3 &= \func[3](\covariatei_1, \covariatei_2, \exogenousi_3) = 0.5 \cdot \covariatei_2 + \covariatei_1 + \exogenousi_3
    \end{cases}
\end{equation*}

\textbf{Triangle non Linear}

\begin{equation*}
    \begin{cases}
        \covariatei_1 &= \func[1](\exogenousi_1) = \exogenousi_1 \\
        \covariatei_2 &= \func[2](\covariatei_1, \exogenousi_2) = 2 \cdot \covariatei_1^2 + \exogenousi_2 \\
        \covariatei_3 &= \func[3](\covariatei_1, \covariatei_2, \exogenousi_3) = \frac{20}{1 + e^{-\covariatei_2^2 + \covariatei_1}} + \exogenousi_3
    \end{cases}
\end{equation*}

Note that all of the above graphs are continuous. In order to obtain their mixed counterparts, we chose to always discretize $\covariatei_3$ by using the following procedure:
\begin{itemize}
    \item First, we standardize $\covariatei_3$.
    \item Then, we sample a discrete ternary variable, with values $\{0, 1, 2\}$, with the following probability vector:
    \begin{itemize}
        \item If $\covariatei_3 <-1$, the probability vector is $[0.8, 0.1, 0.1]$.
        \item If $\covariatei_3 > 1$, the probability vector is $[0.1, 0.1, 0.8]$.
        \item In any other case, the probability vector is $[0.1, 0.8, 0.1]$.
    \end{itemize}
\end{itemize}

Note that the procedure is not deterministic because we sample a categorical random variable, hence, we introduce additional noise when discretizing the covariatei.

\subsection{Training times}
Here we report training times for completeness, in \cref{tab:time_all_experiments}.
\begin{table}[ht]
\scriptsize
\setlength{\tabcolsep}{4pt}
\centering
\begin{tabular}{lcccc}
\hline
& Continuous & Discrete & Sachs Additive & Sachs Nonadditive \\
\hline
KAN & $17.71_{20.5}$ & $37.88_{52.5}$ & $374.90_{53.51}$ & $385.17_{66.42}$ \\
KAAM & $13.79_{5.6}$ & $3.36_{2.98}$ & $356.39_{20.92}$ & $345.11_{49.96}$ \\
ANM & $0.02_{0.0}$ & $0.02_{0.02}$ & $0.28_{0.25}$ & $0.20_{0.21}$ \\
DBCM & $11.19_{8.3}$ & $6.74_{6.28}$ & $502.45_{120.63}$ & $483.95_{134.50}$ \\
CFlow & $3.12_{2.8}$ & $2.92_{2.20}$ & $1.49_{0.51}$ & $9.33_{9.12}$ \\
$\text{KAN}_\text{mix}$ & $-$ & $7.45_{33.93}$ & $-$ & $-$ \\
$\text{KAAM}_\text{mix}$ & $-$ & $10.93_{7.96}$ & $-$ & $-$ \\
\hline
\end{tabular}
\caption{
Training time in seconds across experiments, reported as averages over datasets or runs ($\text{mean}_\text{std}$). Continuous and discrete correspond to the synthetic experiments.
}
\label{tab:time_all_experiments}
\end{table}

\subsection{Biographies}
\begin{IEEEbiography}
[{\includegraphics[width=0.9in,height=1.125in,clip,keepaspectratio]{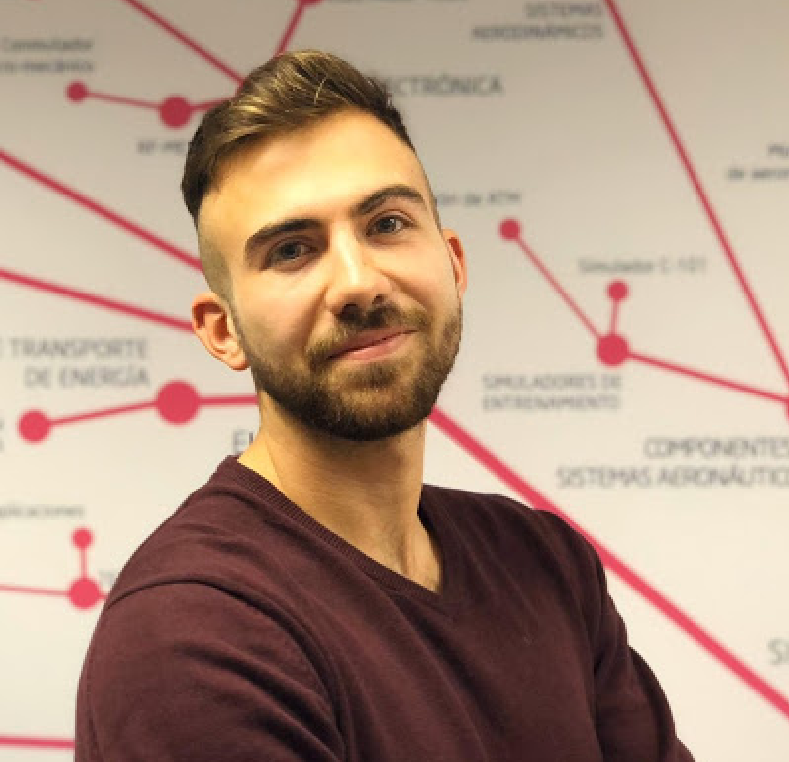}}]{ALEJANDRO ALMODOVAR} received the B.S. degree and the
M.Sc. degree in communications engineering from Universidad Politécnica de Madrid (UPM), Madrid, Spain, in
2020 and 2022, respectively.
Currently, he is a Ph.D. student whose work focuses on
causal inference with deep learning, with
applications in medicine and healthcare, among others.
\end{IEEEbiography}
\begin{IEEEbiography}
[{\includegraphics[width=0.9in,height=1.125in,clip,keepaspectratio]{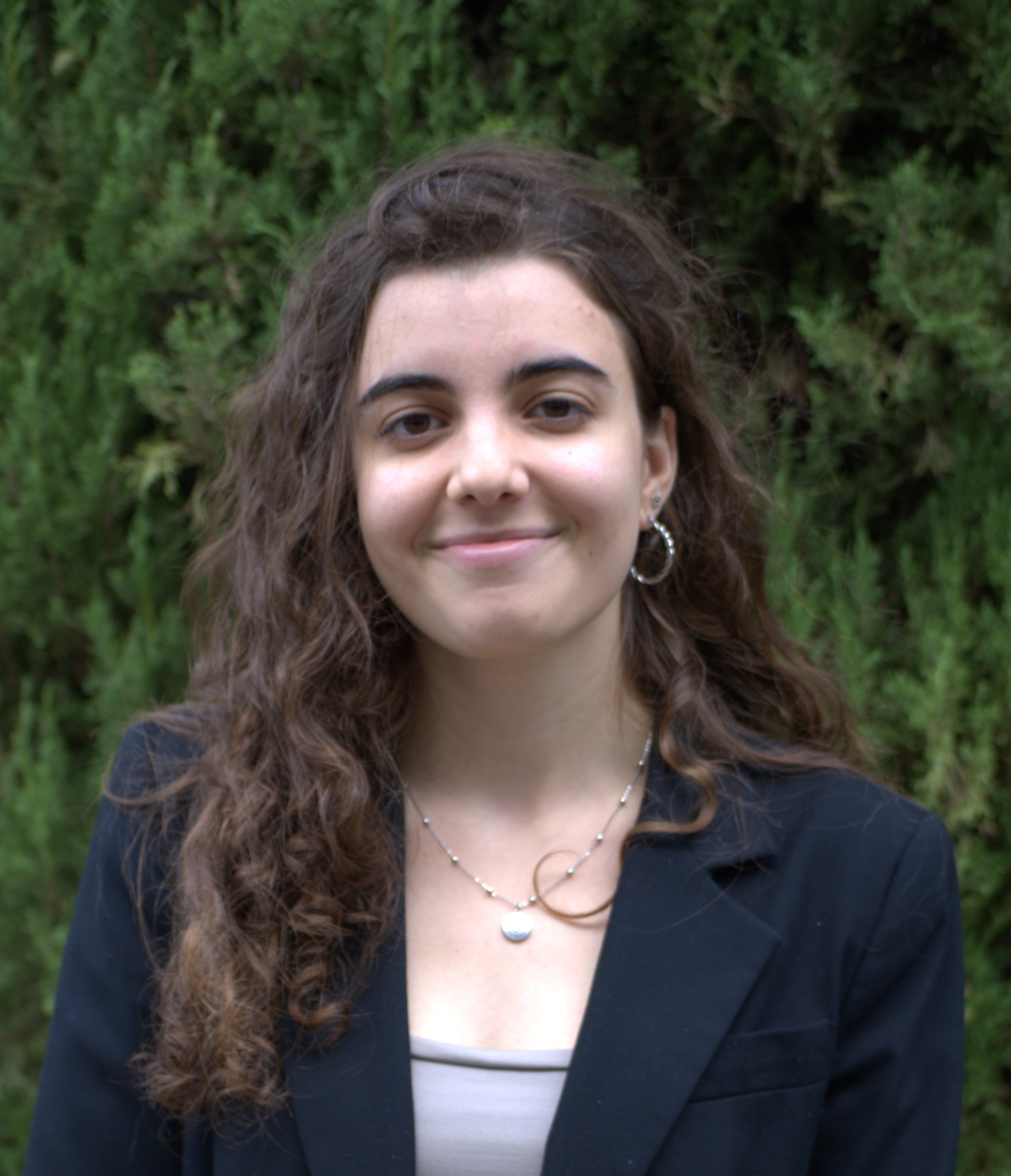}}]{MAR ELIZO}
received the B.S. degree in Computer Engineering from the Universidad Carlos III de Madrid (UC3M) in 2023, and the M.Sc. degree in Artificial Intelligence, Pattern Recognition, and Digital Imaging from the Universitat Politècnica de València (UPV) in 2024. She is currently pursuing a Ph.D. degree at Universidad Politécnica de Madrid (UPM), focusing on deep learning and causal inference applications in healthcare.
\end{IEEEbiography}
\begin{IEEEbiography}
[{\includegraphics[width=0.9in,height=1.125in,clip,keepaspectratio]{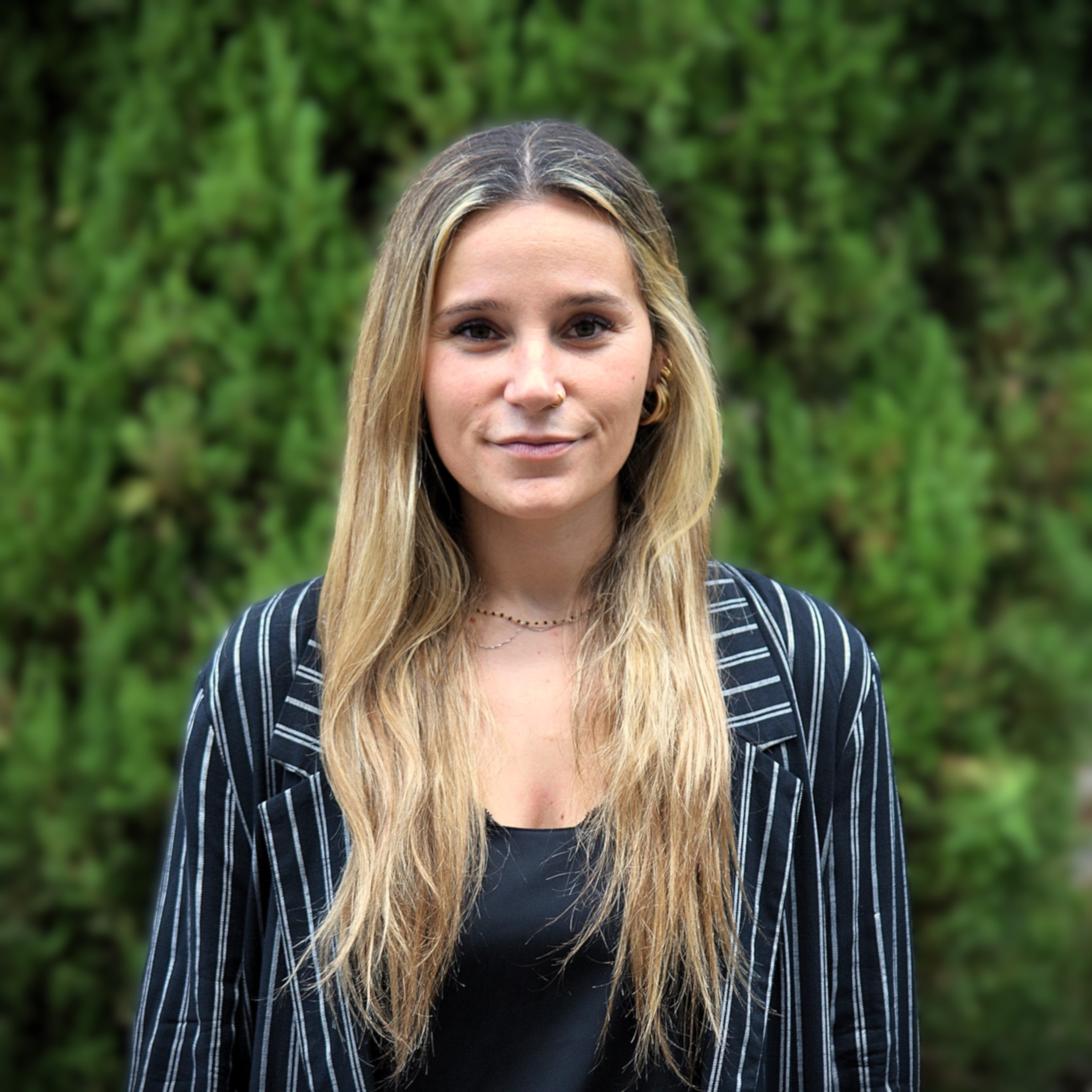}}]{PATRICIA A. APELLÁNIZ}
received the B.S. degree in Telecommunication Engineering from the Universidad Autónoma de Madrid, Spain, in 2018, and the M.Sc. and Ph.D degrees in Telecommunication Engineering from the Universidad Politécnica de Madrid (UPM), Spain (2020 and 2025, respectively). She is currently a researcher and teaching collaborator at UPM. Her research interests include deep learning for healthcare and signal processing.
\end{IEEEbiography}
\begin{IEEEbiography}[{\includegraphics[width=0.9in,height=1.125in,clip,keepaspectratio]{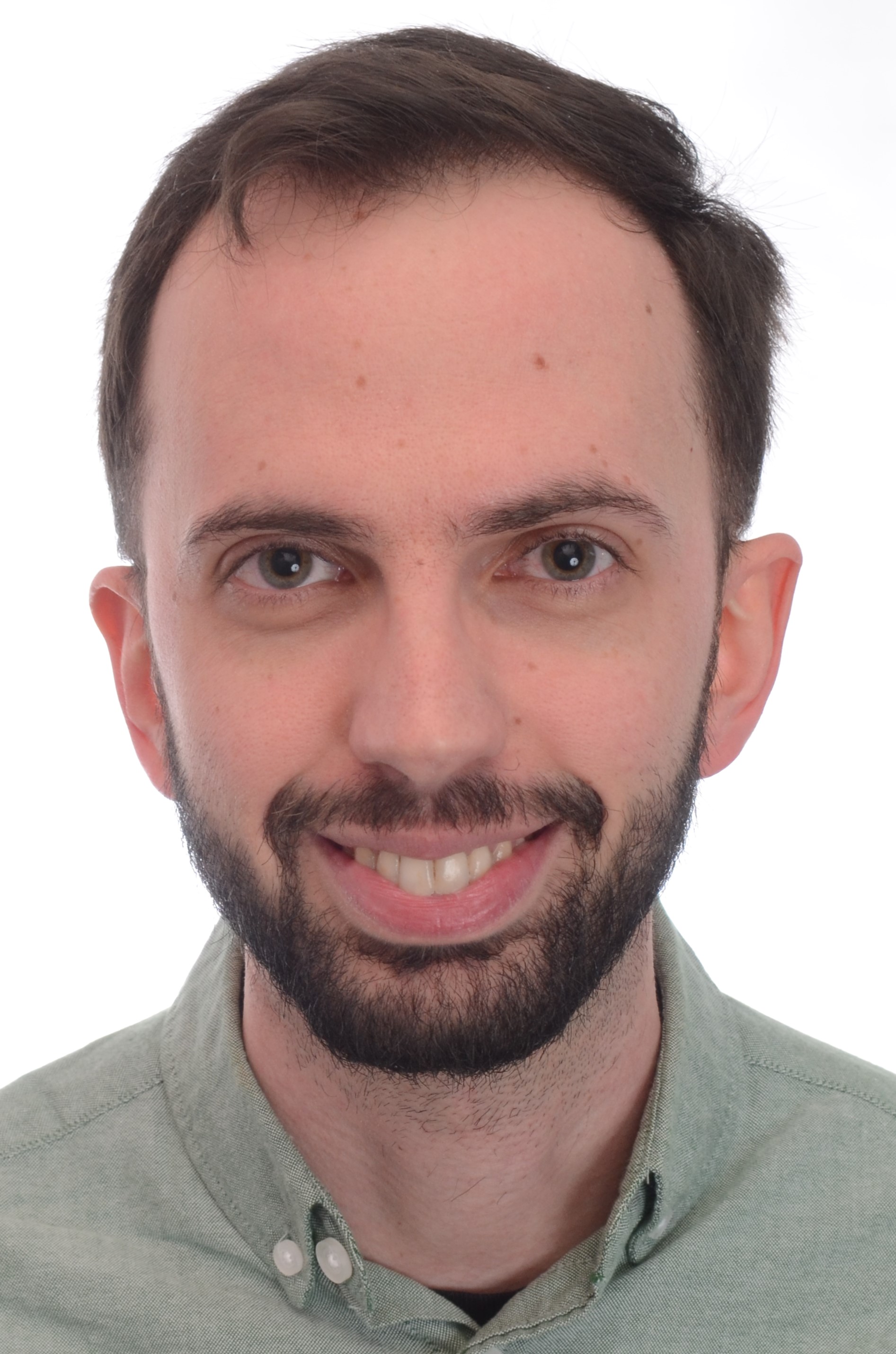}}]{JUAN PARRAS } received the B.S. degree in telecommunications engineering from the Universidad de Jaén in 2014, and the M.Sc. and Ph.D. degrees in telecommunications engineering from the Universidad Politécnica de Madrid (UPM) in 2016 and 2020, respectively. He is currently an Assistant Professor with UPM. His research interests include deep generative models, deep reinforcement learning, game theory, and optimization with health and communications applications.
\end{IEEEbiography}
\begin{IEEEbiography}[{\includegraphics[width=0.9in,height=1.125in,clip,keepaspectratio]{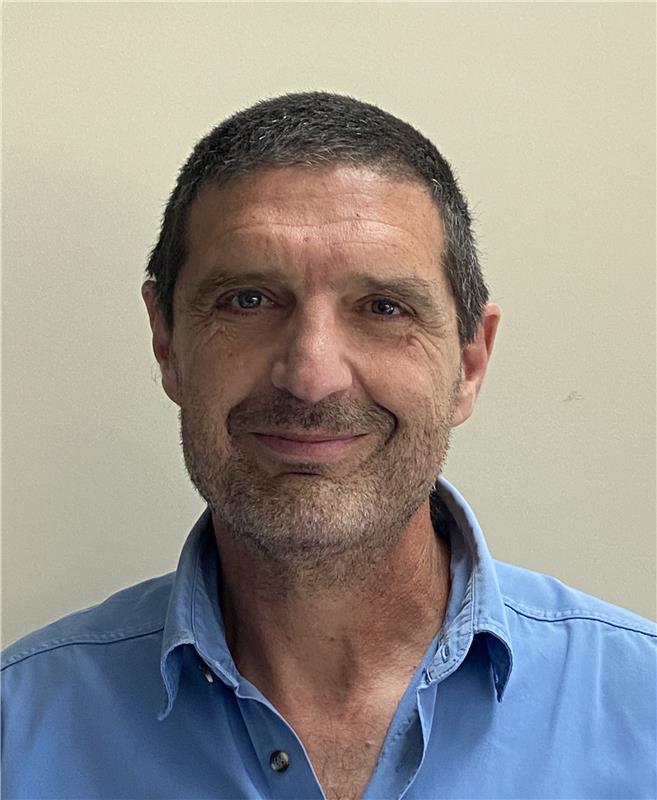}}]{SANTIAGO ZAZO } is currently a Dr. Engineer with the Universidad Politécnica de Madrid (UPM), since 1995. In 1998, he joined UPM, where he is currently a Full Professor in signal theory and communications. More recently, he has been mostly focused on deep learning with medical applications, distributed optimization, game theory, and reinforcement learning. He is the author/co-author of more than 40 journal articles and about 200 conference papers. His main research interest includes signal processing.
\end{IEEEbiography}

\end{document}